\documentclass[anon,12pt]{cleararxiv} 

\title[Reinforcement Learning Under Runtime Confounding]{Reinforcement Learning Under Runtime Confounding:\\Policy Learning and Off-Policy Evaluation}
\usepackage{times}

\clearauthor{%
 \Name{Raphael Kim}  
 \addr Cornell Tech, Cornell University
 \AND
 \Name{Michele Santacatterina} 
 \addr Division of Biostatistics, 
 NYU Grossman School of Medicine
 \AND
 \Name{Ramin Zabih} 
 \addr Cornell Tech, Cornell University
 \AND
 \Name{Keith Barnatchez}  
 \addr Department of Biostatistics, Johns Hopkins University
}

\usepackage{graphicx} 
\usepackage{amsmath,amssymb,mathtools}
\usepackage{enumitem}
\usepackage{booktabs}

\newcommand{\qt}{\quad\quad\quad}

\usepackage[capitalise]{cleveref}

\newcommand{\R}{\mathbb{R}}
\newcommand{\E}{\mathbb{E}}

\newcommand{\St}{\mathcal{S}}
\newcommand{\A}{\mathcal{A}}

\newcommand{\indep}{\perp \!\!\! \perp}
\usepackage{graphicx,tikz}
\usetikzlibrary{arrows.meta}

\newcommand{\pr}{\mathbb{P}}
\newcommand{\I}[1]{\mathbb{I}[#1]}

\newcommand{\K}{\mathcal{K}}

\newcommand{\V}{\mathbb{V}}

\newcommand{\HX}{H^{X}}
\newcommand{\RTC}{\text{RTC}\xspace}

\begin{document}

\maketitle

\begin{abstract}%
Reinforcement learning (RL) is increasingly used to train robust, adaptable systems across healthcare, e-commerce, and robotics. In these settings, legal, ethical, or logistical concerns may preclude certain key variables that drive decisions from being collected during deployment time. This gives rise to a setting known as runtime confounding: the logged data hold full observability, but deployment time lacks knowledge of certain variables, known as runtime confounders. Unlike single timepoint and non-adaptive settings, runtime confounding under reinforcement learning has the potential to propagate over time, as earlier decisions affect future states, observations, and action assignments, leading to biased learning. We address these problems with three contributions: we (1) develop policy learning algorithms under runtime-confounding, (2) develop semi-parametric efficient off-policy estimation procedures, and (3) provide diagnostics for quantifying the price of runtime confounding, and a policy learning algorithm under a budget that selects which runtime confounders to collect to reduce this gap.
We support our theory with extensive simulations, demonstrating the efficacy of our approach in policy evaluation and learning relative to standard methods.
\end{abstract}


\section{Introduction}

Sequential decision-making rules are often deployed in settings with less information than was available during training. For example, a clinical decision-support tool may be trained on records that include laboratory results, clinician notes, or socioeconomic variables, but some of these variables that relate to sensitive subgroup classification, such as socioeconomics, may not be allowed as inputs when the tool is deployed to protect against discriminative profiling.
This problem can occur in other settings with high-stakes decision-making, such as adjusting credit limits or child welfare screening. Despite having knowledge of all variables that influence outcomes in historical logs, during deployment time, one may not have access to an individual's protected attributes for ethical, legal, or logistical reasons. 
More formally, these variables observed in historical data but are unavailable at deployment are known as \textit{runtime confounders}, and the setting of misaligned logging data and deployment data is referred to as \emph{runtime confounding} \citep{coston2020counterfactual}.

While a large part of reinforcement learning (RL) literature studies a related problem of unobserved confounding, to the best of our knowledge, no work studies runtime confounding precisely. 
More precisely, past work assume both logging and deployment lack confounders, and handle the non-identifiability with a range of approaches including causal bounds and online identification \citep{bareinboim2015bandits,zhang2019near}, sensitivity analysis \citep{kallus2018confounding,namkoong2020off}, and proximal or partially observable formulations \citep{tennenholtz2020off,bennett2021proximal,uehara2023future}. 
In runtime confounding, the difficulty lies in the fact that the deployed policy can only condition on a variable set that does not include the runtime confounders, leading to a different policy class than that used in the historical data and violations to sequential randomization during deployment. 
\citet{coston2020counterfactual} address runtime confounding in single-period counterfactual prediction, and recent work studies a conformal prediction approach in the same setting \citep{barnatchez2026conformal}.

We address this gap by proposing a semiparametric-efficient framework for policy learning and off-policy evaluation under runtime confounding. The key observation is that the runtime restriction enters only through the numerator of an importance ratio. Hence, one may take a reweighting approach to recover the value of any policy in the runtime-confounding policy class, with all nuisances estimated from the full data. As a result, reinforcement learning under runtime confounding can fit within the ideas of doubly robust off-policy evaluation \citep{jiang2016doubly,thomas2016data,KallusUehara2020DRL,van2015targeted} with cross-fitted nuisances \citep{chernozhukov2018double}, as well as the literature on dynamic treatment regimes for restricted policies \citep{luedtke2016super,sakaguchi2025estimation} and the broader literature on policy learning with observational data \citep{athey2021policy}.

Our contributions are as follows.
\begin{enumerate}
  \item \emph{Policy learning.} We propose a doubly robust policy learning algorithm for reinforcement learning under runtime confounding over a class of runtime deployable policies (Algorithm~\ref{alg:rtc}). Its regret relative to the full-information optimum is an irreducible runtime regret gap and efficient rates (Theorem~\ref{thm:regret}).
  \item \emph{Semiparametric-Efficient Off-Policy Evaluation.} For a fixed runtime policy, we propose a doubly-robust estimator that is asymptotically normal and attains the semiparametric efficiency bound, yielding valid confidence intervals and testing procedures (Theorem~\ref{thm:clt}).
  \item \emph{Testing the Price of Runtime Confounding, and Improving the Runtime Gap under a Budget.} Because the runtime confounders are logged, we can test the price of runtime confounding. We propose a test for the irreducible runtime regret gap, giving a decision-relevant go/no-go criterion for candidate runtime confounders, and an accompanying policy learning algorithm for what runtime confounders to collect under budget constraints.
\end{enumerate}

\section{Problem Setting}
\label{sec:problem-setting}

\subsection{Notation}
\label{sec:notation}
We consider a setting where we have complete observation logging data available over $i \in[n]$ sample trajectories and $t \in [T]$ timepoints. Trajectories are denoted by
\[
    O=(S_1,A_1,R_1,\dots,S_T,A_T,R_T),\qquad S_t=(X_t,U_t)\in\St
\]
where $S_t$ is the full time-$t$ state consisting of $X_t$ (available at deployment) and $U_t$ (recorded in the logs, but unavailable at runtime), $A_t\in \A$ is the action, and $R_t \in \R$ is the reward, bounded by $R_{max}$. Define the histories
\begin{align*}
    H_t &= (S_1,A_1,R_1,\dots,S_{t-1},A_{t-1},R_{t-1},S_t), \quad   \HX_t = (X_1,A_1,R_1,\dots,X_{t-1},A_{t-1},R_{t-1},X_t) 
\end{align*}
so that $\HX_t$ is the runtime version of the history, or $H_t$ with each full state $S_s$ replaced by its runtime component $X_s$. 
The observed data are generated by a logging policy $\pi_b(\cdot \mid H_t)$ which we allow to be known or unknown by the researcher, where rewards and states are generated according to $R_t\sim p_t(\cdot\mid H_t,A_t)$ and $S_{t+1}\sim p_t(\cdot\mid H_t,A_t,R_t)$, respectively. Importantly, during deployment $U_t$ is unavailable so that newer target policies which depend on $U_t$ are infeasible. Figure \ref{fig:setting} outlines our problem setting. 

Letting $\bar a_s = a_1,\ldots,a_s)$ denote a fixed sequence of actions, we let $R_t(\bar a_t)$ denote the counterfactual reward under action sequence $\bar a_t$, and similarly let $S_t(\bar a_{t-1})$ denote the counterfactual state at time $t$ that would be observed under action sequence $\bar a_{t-1}$.

\begin{figure}
    \centering
    \resizebox{0.8\linewidth}{!}{%
\begin{tikzpicture}[
  x=1cm,y=1cm,font=\small,
  box/.style={
    draw=black!65,rounded corners=2pt,line width=.6pt,
    align=center,minimum height=.64cm,inner sep=4pt
  },
  observed/.style={box,fill=blue!5},
  unavailable/.style={box,dashed,fill=black!4},
  policy/.style={box,fill=blue!9,text width=2.55cm},
  environment/.style={box,fill=black!3,text width=2.55cm},
  link/.style={
    -{Stealth[length=1.7mm]},line width=.7pt,draw=black!75
  }
]
\path[use as bounding box] (-.1,.45) rectangle (14.7,-3.65);

\node[font=\small\bfseries] at (3.45,.15)
  {Historical data};

\node[observed,text width=1.85cm] (hx-log) at (1.15,-.85)
  {Runtime history\\$H_t^X$};
\node[observed,text width=2.2cm] (u-log) at (1.15,-2.5)  {Logged\\\mbox{confounders}\\$\bar U_t$};
\node[policy] (policy-log) at (5.0,-.85)
  {Logging policy\\$\pi_b(\cdot\mid H_t)$};
\node[environment] (env-log) at (5.0,-2.5)
  {Rewards / transitions\\$R_t,\ S_{t+1}$};

\draw[link] (hx-log.east) -- (policy-log.west);
\draw[link] (u-log.north east) -- (policy-log.south west);
\draw[link] (u-log.east) -- (env-log.west);
\draw[link] (policy-log.south)
  -- node[right] {$A_t$} (env-log.north);


\draw[black!20,line width=.5pt] (7.25,.35) -- (7.25,-3.55);

\begin{scope}[xshift=7.6cm]
\node[font=\small\bfseries] at (3.45,.15)
  {Deployment};

\node[observed,text width=1.85cm] (hx-run) at (1.15,-.85)
  {Runtime history\\$H_t^X$};
\node[unavailable,text width=1.85cm] (u-run) at (1.15,-2.5)
  {Unavailable to policy\\$\bar U_t$};
\node[policy] (policy-run) at (5.0,-.85)
  {Target policy\\$\pi_t(\cdot\mid H_t^X)$};
\node[environment] (env-run) at (5.0,-2.5)
  {Rewards / transitions\\$R_t,\ S_{t+1}$};

\draw[link] (hx-run.east) -- (policy-run.west);
\draw[link] (u-run.east) -- (env-run.west);
\draw[link] (policy-run.south)
  -- node[right] {$A_t$} (env-run.north);

\end{scope}
\end{tikzpicture}%
}

    \caption{Problem setting. $U_t$ is observed in the historical data, but unmeasured during deployment and in turn unavailable to the target policy.}
    \label{fig:setting}
\end{figure}

\subsection{Target Parameters and Identification}
\label{sec:target-params}

Our target is the $\gamma-$discounted value function under a fixed policy $\pi$: 
\begin{equation}\label{eq:obj}
    J(\pi) =  \E_\pi\Bigl[ \sum_{t=1}^{T} \gamma^{t-1} R_t \Bigr],
\end{equation}
where $\gamma \in [0,1]$ is fixed by the researcher. In RL, one typically seeks to learn an optimal policy, and evaluate fixed policies, which act on the full history $H_t$, or current contexts $S_t$ and any relevant past information. What makes runtime confounding unique compared to standard RL settings is that we do not observe $U_t$ at runtime, so policies that require $U_t$ are infeasible.  In turn, we instead aim to perform inference and learning over the policy class which acts only over the variables available at runtime:
\[
    \Pi_X=\bigl\{\pi=(\pi_t)_{t\le T}:\mbox{ } \pi_t(\cdot\mid\HX_t), \mbox{ } 
    \sigma(\HX_t)\text{-measurable}\bigr\}
\]

Importantly, $J(\pi)$ is not immediately identified for policies $\pi$ other than the logging policy $\pi_b$ that generated the observed data. We assume the following identification conditions to identify $J(\pi)$ for any $\pi\in\Pi_X$.
\begin{enumerate}[label=(\textbf{Id\arabic*}),ref=(\textbf{Id\arabic*}),leftmargin=*]
  \item\label{ass:seq} \textit{Sequential ignorability.} 
  For all $t \in [T]$, and for any generic action sequence $\bar a_T = (a_1, \dots ,a_T)$, 
  \[
  R_t(\bar a_t), \ldots, R_T(\bar a_T), S_{t+1}(\bar a_t), \ldots, S_T(\bar a_{T-1}) \indep A_t \mid H_t
  \]
  \item \label{ass:consistency} \textit{Consistency}: $S_t = S_t(\bar A_{t-1})$ and $R_t = R_t(\bar A_t)$ for all $t$.
  \item\label{ass:overlap} \textit{Overlap.} $\pi_b(a\mid H_t)\ge\eta_g>0$ for all $a\in\A$ and $t\in[T]$.
\end{enumerate}
\ref{ass:seq} -- \ref{ass:overlap} are standard identification assumptions in policy learning \citep{murphy2001marginal,zhang2013robust} and reinforcement learning \citep{KallusUehara2020DRL}. 

With these identification assumptions, one can identify $J(\pi)$ from the observed data. Let the superscript $\RTC$ denote the runtime version of any fixed object. Then define the runtime importance ratio as:
\begin{equation}\label{eq:rho}
    \rho^{\RTC}_{1:t} =\prod_{s=1}^{t}\lambda^{\RTC}_s, \qquad  \mbox{for }
    \lambda^{\RTC}_s=\frac{\pi_s(A_s\mid\HX_s)}{\pi_b(A_s\mid H_s)}
\end{equation}
Here, the importance weight takes a particular form: In contrast to the non \RTC setting which uses $\rho_{1:t} = \prod_{s=1}^{t} \frac{\pi_s(A_s\mid H_s)}{\pi_b(A_s\mid H_s)}$, the \RTC importance weight has the numerator correspond to runtime-confounded quantities while the denominator is the full information policy. For $t = T,\ldots,1$, we then recursively define the Q- and Value function.
\begin{align*}
    Q_t^\pi(H_t,a) &= \E_{\pi_b}[R_t + \gamma V_{t+1}^\pi(H_{t+1}) \mid H_t, A_t=a] \\
    V_t^\pi(H_t) &= \sum_a \pi_t(a \mid H_t^X)Q_t^\pi(H_t,a)
\end{align*}
where $V_{T+1}^\pi=0$. Again, the primary difference is the action density that we marginalize over --- a runtime policy.

Under Assumptions \ref{ass:seq}-\ref{ass:overlap}, $J(\pi)$ is identified by the observed data functionals:
\begin{equation}
    J(\pi) = \E_{\pi_b}[V_1^\pi(H_1)] = \E_{\pi_b} \left[ \sum_{t=1}^T \gamma^{t-1} \rho_{1:t}^\RTC R_t \right].
    \label{eq:id}
\end{equation}
Importantly, the identifying functional in \eqref{eq:id} does not discard the historical $U_t$. Discarding variables that unavailable at runtime is a common practice in data science and machine learning \citep{coston2020counterfactual}, but will tend to yield invalid identifying functionals when the discarded variables are truly confounders. Specifically, focusing on the importance-weighted identifying functional in \eqref{eq:id},
if we were to to instead use the importance weights
\[
\rho^{X}_{1:t} =\prod_{s=1}^{t}\lambda^{X}_s,  \qquad\mbox{where} \qquad \lambda_s^X = \frac{\pi_s(A_s \mid H_s^X)}{\pi_b(A_S \mid H_s^X)}
\]
the resulting identifying functional would generally not equal $J(\pi)$. In turn, standard methods for both policy learning and inference \citep{KallusUehara2020DRL} that naively estimate this functional would target the wrong parameter inducing unintended bias.

\section{Efficient Estimation and Policy Learning}
Both policy learning and off-policy estimation will build on the efficient influence function (EIF) for the identifying functional \eqref{eq:id} of the target parameter $J(\pi)$. EIFs are crucial ingredients for constructing estimators of statistical parameters, and enable valid $\sqrt n$ asymptotic inference even when estimators of underlying nuisance functions themselves converge at slower rates \citep{kennedy2024semiparametric}. 
Throughout, we will estimate the EIF of $J(\pi)$ using nuisance models for the propensity and the conditional outcome model, denoted by $\eta=(\pi_b(\cdot\mid H_t),\{Q^\pi_t(H_t,\cdot)\})$. Writing $\delta^\pi_t=R_t+\gamma V^\pi_{t+1}(H_{t+1})-Q^\pi_t(H_t,A_t)$ for the Bellman residual, the efficient influence function of $J(\pi)$ at fixed $\pi\in\Pi_X$ is given by the following Theorem:

\begin{theorem}[Efficient Influence Function under Runtime-Confounding]\label{thm:EIF}
Fix $\pi\in\Pi_X$. Under our assumptions \labelcref{ass:seq}-\labelcref{ass:overlap}, the efficient influence function of $J(\pi)$ is
\begin{equation}
  \varphi(O;\pi,\eta) = V^\pi_1(H_1)-J(\pi) +\sum_{t=1}^{T}\gamma^{t-1}\rho^{\RTC}_{1:t}\,\delta^\pi_t
  \label{eq:eif}
\end{equation}
where $\E[\varphi(O ; \pi,\eta)]=0$. The proof is found in \cref{pf:EIF}.
\end{theorem}

The EIF in \eqref{eq:eif} is analogous to the EIF considered by \cite{KallusUehara2020DRL} for non-Markov decision processes, where here the importance weights are restricted to accommodate policies $\pi \in \Pi_X$. 

Theorem \ref{thm:EIF} motivates an estimator based on the EIF $\varphi(O; \pi, \eta)$, where we use cross-fitting \citep{chernozhukov2018double} to ensure the resulting asymptotic distribution remains tractable. Specifically, letting $\hat\eta=(\hat g,\{\hat q^\pi_t\})$ denote nuisances cross-fitted on folds not containing unit $i$, we target the population-level moment $\E[\varphi(O ; \pi,\eta)]=0$ by proposing the estimator
\begin{equation}
  \hat J(\pi)=\frac1n\sum_{i=1}^n\Bigl\{\widehat V^\pi_1(H^{(i)}_1)
  +\sum_{t=1}^T\gamma^{t-1}\hat\rho^{\RTC,(i)}_{1:t}\hat\delta^{\pi,(i)}_t\Bigr\}.
  \label{eq:Jhat}
\end{equation}
The full estimation procedure is outlined in Algorithm \ref{alg:rtc}.

\begin{algorithm2e}[t]
\DontPrintSemicolon
\SetAlCapHSkip{0pt}
\caption{Doubly robust policy learning and inference under runtime confounding}
\label{alg:rtc}
\SetKwFunction{FValue}{Value}
\SetKwProg{Fn}{Function}{:}{}
\KwIn{Logged trajectories $\{O_i\}_{i=1}^n$ with $S_{it}=(X_{it},U_{it})$; policy class $\Pi\subseteq\Pi_X$; folds $K$; overlap constant $\eta_g$; $C_Q=R_{\max}\sum_t\gamma^{t-1}$}
\KwOut{$\hat\pi$, $\hat J(\cdot)$, confidence intervals}
Split $[n]$ into folds $\mathcal I_1,\dots,\mathcal I_K$\;
\tcp{Propensities (full information), fitted on $\mathcal I_{-k}$}
\For{$k=1,\dots,K$}{
  \For{$t=1,\dots,T$}{
    Fit $\hat g^{(-k)}_t(a\mid H_t)$ by regressing $A_t$ on $H_t$\;
  }
}
\tcp{Cross-fitted value estimate for $\pi\in\Pi_X$}
\Fn{\FValue{$\pi$}}{
  \For{$k=1,\dots,K$}{
    Set $\hat V^{(-k)}_{T+1}\equiv0$\;
    \tcp{Fitted-$Q$ evaluation on $\mathcal I_{-k}$}
    \For{$t=T,\dots,1$}{
      Fit $\hat q^{\pi,(-k)}_t(H_t,a)$ by regressing $R_t+\gamma\hat V^{(-k)}_{t+1}(H_{t+1})$ on $(H_t,A_t)$\;
      $\hat V^{(-k)}_t(H_t)\leftarrow\sum_a\pi_t(a\mid\HX_t)\hat q^{\pi,(-k)}_t(H_t,a)$\;
    }
    \tcp{Evaluation on the held-out fold $\mathcal I_k$}
    \For{$i\in\mathcal I_k$}{
      $\hat\rho^{\RTC}_{i,1:t}\leftarrow\prod_{s\le t}\frac{\pi_s(A_{is}\mid\HX_{is})}{\hat g^{(-k)}_s(A_{is}\mid H_{is})}$ for $t\in[T]$\;
      $\hat\delta^\pi_{it}\leftarrow R_{it}+\gamma\hat V^{(-k)}_{t+1}(H_{i,t+1})-\hat q^{\pi,(-k)}_t(H_{it},A_{it})$\;
      $\psi_i(\pi)\leftarrow\hat V^{(-k)}_1(H_{i1})+\sum_{t=1}^T\gamma^{t-1}\hat\rho^{\RTC}_{i,1:t}\hat\delta^\pi_{it}$\;
    }
  }
  \KwRet{$\hat J(\pi)=\frac{1}{n}\sum_{i=1}^n\psi_i(\pi)$ and $\hat\sigma^2(\pi)=\frac{1}{n}\sum_{i=1}^n\bigl(\psi_i(\pi)-\hat J(\pi)\bigr)^2$}\;
}
\tcp{Learning and inference}
\textbf{Policy learning:} $\hat\pi\leftarrow\arg\max_{\pi\in\Pi}$ \FValue{$\pi$}\;
\textbf{Inference for a fixed $\pi$:} report $\hat J(\pi)\pm z_{1-\alpha/2}\hat\sigma(\pi)/\sqrt{n}$\;
\end{algorithm2e}

\paragraph{Error Decomposition with the EIF. }
We now describe the high-level decomposition used to produce our policy learning and off-policy evaluation guarantees. Let $\pr_n f=n^{-1}\sum_{i=1}^n f(O_i)$ denote the empirical measure, $\widehat\varphi^\pi$ denote $\varphi^\pi$ with cross-fitted nuisance estimates, and $\Delta g_t=\hat g_t-\pi_b$ and $\Delta q^\pi_t=\hat q^\pi_t-Q^\pi_t$ denote estimation error of each respective nuisance function. To analyze $\hat J(\pi)$, we decompose our error into the following terms:
\begin{equation}
  \hat J(\pi)-J(\pi)
  =\underbrace{(\pr_n-\E)\varphi^\pi}_{\text{(E)}} +\underbrace{\mathrm{Rem}(\hat\eta;\pi)}_{\text{(R)}} +\underbrace{(\pr_n-\E)(\widehat\varphi^\pi-\varphi^\pi)}_{\text{(C)}}
  \label{eq:expand}
\end{equation}
where the remainder is given by:
\begin{equation}
  \mathrm{Rem}(\hat\eta;\pi) =\sum_{t=1}^{T}\gamma^{t-1} \E\Bigl[\hat\rho^{\RTC}_{1:t-1}\sum_{a}\pi_t(a\mid\HX_t) \frac{\Delta g_t(a\mid H_t)}{\hat g_t(a\mid H_t)} \Delta q^\pi_t(H_t,a)\Bigr]
  \label{eq:exact}
\end{equation}

This is a standard decomposition in the semiparametric literature \citep{kennedy2024semiparametric}, where term (E) is sampling error, term (R) denotes nuisance bias, and term (C) is controlled by cross-fitting. This decomposition to establish regret bounds and conditions under which $\hat J(\pi)$ attains asymptotic normality.

Before introducing our results, we introduce the following \textbf{estimation assumptions}, which are sufficient to control the remainder term in \eqref{eq:exact}.
\begin{enumerate}[label=(\textbf{Est\arabic*}),ref=(\textbf{Est\arabic*}),leftmargin=*]
  \item\label{ass:product} \textit{Product rate at $\pi$.} $\sum_{t=1}^T\lVert\Delta g_t\rVert_2\lVert\Delta q^\pi_t\rVert_2=o_P(n^{-1/2})$.
  \item\label{ass:consist} \textit{Consistency at $\pi$.} $\lVert\Delta g_t\rVert_2+\lVert\Delta q^\pi_t\rVert_2=o_P(1)$ for each $t\in[T]$.
  \item\label{ass:complex} \textit{Complexity.} $\Pi_X$ has finite uniform entropy integral bounded by $\mathcal J<\infty$
\end{enumerate}
\labelcref{ass:product}-\labelcref{ass:consist} are standard rate assumptions in semiparametric inference supporting the use of non-Donsker function classes in \textit{estimation}. \labelcref{ass:complex} is an entropy condition on the runtime policy class, rather than the nuisance function estimator, used for regret analysis. In the coming sections, we establish off-policy evaluation (\labelcref{sec:ope}) and policy learning (\labelcref{sec:policy}) results based on these assumptions.

\subsection{Efficient Off-Policy Evaluation}\label{sec:ope}

We now fix $\pi\in\Pi_X$ and give conditions under which $\hat J(\pi)$ in \eqref{eq:Jhat} is asymptotically normal and efficient based on \cref{thm:EIF}.

\begin{theorem}[Asymptotic normality and efficiency]\label{thm:clt}
Under our identification assumptions \labelcref{ass:seq}-\labelcref{ass:overlap}, and estimation assumptions \labelcref{ass:product}-\labelcref{ass:consist},
\[
  \sqrt n\bigl(\hat J(\pi)-J(\pi)\bigr)\rightsquigarrow\mathcal N(0,\sigma^2_{\RTC})
\]
where $$\sigma^2_{\RTC}=\V\bigl[V^\pi_1(H_1)\bigr]+\sum_{t=1}^T\gamma^{2(t-1)}\E\bigl[(\rho^{\RTC}_{1:t})^2\V[\delta^\pi_t\mid H_t,A_t]\bigr] $$
is the semiparametric efficiency bound. Moreover $\hat\sigma^2_{\RTC}=\pr_n(\widehat\varphi^\pi)^2$ is consistent, so $\hat J(\pi)\pm z_{1-\tfrac \alpha 2}\frac{\hat\sigma_{\RTC}}{\sqrt n}$ is an asymptotically valid $(1-\alpha)$ confidence interval.
The proof is found in \labelcref{pf:clt}.
\end{theorem}

This result gives us a way to characterize uncertainty for runtime-confounded policies, and do valid testing procedures. The finding is analogous to the results from \cite{KallusUehara2020DRL}, but \RTC adjusted.

\subsection{Policy Learning under Runtime-Confounding}\label{sec:policy}
Let $\pi^\star=\arg \max_\pi J(\pi)$ over the unrestricted (full-information) policies, and $\pi^\star_X=\arg\max_{\pi\in\Pi_X}J(\pi)$.
The learned policy maximizes the cross-fitted estimator \eqref{eq:Jhat} over runtime policies,
\[
  \hat\pi=\arg\max_{\pi\in\Pi_X}\hat J(\pi)
\]
so that $\hat\pi$ takes only $\HX_t$ as input and is deployable.

\begin{theorem}[Regret under runtime confounding]\label{thm:regret}
Under the identification assumptions \labelcref{ass:seq}--\labelcref{ass:overlap}, \labelcref{ass:product} at $\hat\pi$ and $\pi^\star_X$, and \labelcref{ass:complex},
\begin{equation} 
  J(\pi^\star)-J(\hat\pi)\leq\Delta_{\RTC}+O_P\Bigl(\frac{C_\rho\mathcal J}{\sqrt n}\Bigr)+o_P(n^{-1/2})
  \label{eq:regret}
\end{equation}
where $\Delta_{\RTC}=J(\pi^\star)-J(\pi^\star_X)$ is the value lost by restricting to runtime information, with $\pi^\star$ and $\pi^\star_X$ the optimal policies in $\Pi$ and $\Pi_X$, respectively.
Here, $C_\rho\lesssim(R_{\max}+C_Q)\sum_{t=1}^T\gamma^{t-1}\eta_g^{-t}$, where we recall that $\eta_g$ is our positivity constraint for the logger. The proof is found in \labelcref{pf:regret}.\end{theorem}

The result states that our learned policy, under sufficient regularity conditions, follows the same finding of $\sqrt{n}$ rates with an irreducible runtime gap $\Delta_{\RTC}$.

\section{Diagnosing and Improving the Runtime Confounding Gap}\label{sec:testing}

One may ask the gap we are losing in the decision-making process by using variables available only at runtime. This can inform future experimentation or provide an interpretable gap when it is unavoidable for domain experts. We lower-bound and test the \textit{irreducible runtime gap} $\Delta_{\RTC}=J(\pi^\star)-J(\pi^\star_X)\geq0$, which is the value lost by withholding $\bar U_t$ at deployment.  Because $\Delta_{\RTC}$ is a difference of maxima, inference on this parameter can be non-regular \citep{luedtke2016statistical}. To address this difficulty, we employ a sample-splitting approach.  Specifically, we split the sample into parts $A$ and $B$. On $A$, we learn a full-information policy $\hat\pi_S$ on $H_t$ and a runtime policy $\hat\pi_X\in\Pi_X$ by Algorithm~\ref{alg:rtc}. On $B$, estimate the runtime gap by 
\[
  \hat\Delta_{\RTC}=\hat J_B(\hat\pi_S)-\hat J_B(\hat\pi_X)
\]
By Theorem~\ref{thm:clt}, $\sqrt{n_B}\bigl(\hat\Delta_{\RTC}-[J(\hat\pi_S)-J(\hat\pi_X)]\bigr)$ is asymptotically normal.

Assuming $J(\pi^\star)-J(\hat\pi_S)=o_P(n^{-1/2})$ and $J(\pi^\star_X)-J(\hat\pi_X)=o_P(n^{-1/2})$, 
$$\hat\Delta_{\RTC}\pm z_{1-\alpha/2}\hat\sigma/\sqrt{n_B}$$ 
is an asymptotically valid $(1-\alpha)$ confidence interval for $\Delta_{\RTC}$, and we can test $\mathbb{H}_0:\Delta_{\RTC}=0$, by 
$p=2\bigl(1-\Phi(\sqrt{n_B}|\hat\Delta_{\RTC}|/\hat\sigma)\bigr)$. Under $\mathbb{H}_0$ this can fail if the two policies coincide so their scores are equal, so we impose $\V[\varphi^{\hat\pi_S}-\varphi^{\hat\pi_X}]\geq c>0$.

\paragraph{Improving the Runtime Gap: Policy Learning under Budgeted Runtime Confounding.}
The ideas above provide a path for understanding which runtime-unavailable variables are worth acquiring when we have the ability to choose certain variables to collect, each with a cost.

Let $\K\subseteq\{1,\dots,d_U \}$ denote  indices for a set of candidate runtime confounders that we wish to collect. For each $j\in\K$, let $c(U_j)$ be the cost of collecting variable $U_j$, with $c(\K)=\sum_{j\in\K}c(U_j)$, and $\pi_{X,\K}\in\Pi_{X,\K}$ denote policies acting on $(\HX_t,\{U_{t,j}\}_{j\in\K})$, with $\Pi_{X,\emptyset}=\Pi_X$.
The quantity relevant to deployment is the runtime gap
\begin{align*}
  \theta_\K=J(\hat\pi_{X,\K})-J(\hat\pi_X),
\end{align*}
denoting the gain from deploying the augmented runtime policy rather than the runtime policy. In a similar procedure to above, we may derive a test statistic and confidence range for this value assuming regularity conditions. With $c(\K)$ denoting the discounted cost of acquiring $U_\K$ at deployment, the net benefit $\theta_\K-c(\K)$ has asymptotic $(1-\alpha)$ confidence interval $\hat\theta_\K-c(\K)\pm z_{1-\alpha/2}\hat\sigma_\K/\sqrt{n_B}$. Then, a rule of thumb for go/no-go is to acquire $U_\K$ (go) when our point estimate is positive, and do not acquire it (no-go) otherwise.

When several candidate subsets are considered under a budget $b$, we select the subset and policy jointly. Let $N_b$ be the number of feasible subsets with $c(\K)\leq b$. On split $A$, we compute
\begin{align*}
  (\hat \K,\hat \pi_{X,\hat\K}) \in \arg \max_{\K:c(\K)\leq b, \pi \in \Pi_{X,\K}}\bigl \{\hat J_A(\pi)-c(\K) \bigr \},
  \qquad \hat\pi_X \in \arg \max_{\pi \in \Pi_X} \hat J_A(\pi),
\end{align*}
where $\hat J_A$ denotes the cross-fitted estimator computed on $A$. 
On split $B$, with $\hat\K$, $\hat\pi_{X,\hat\K}$ and $\hat\pi_X$ held fixed, we apply the go/no-go rule: acquire $U_{\hat\K}$ if $\hat J_B(\hat\pi_{X,\hat\K})-\hat J_B(\hat\pi_X)-c(\hat\K)>0$, and deploy $\hat\pi_X$ otherwise. One can additionally apply more a more conservative confidence interval based go/no-go rule, acquiring variables only when the lower confidence bound on net benefit exceeds zero.

\begin{theorem}[Regret under a budget]\label{thm:budget-regret}
Let $(\K^\star,\pi^\star_{X,\K^\star})$ denote the maximizer of $J(\pi)-c(\K)$ over $c(\K)\leq b$, $\pi\in\Pi_{X,\K}$.
Suppose the conditions of Theorem~\ref{thm:regret} hold for each $\Pi_{X,\K}$ with $c(\K)\leq b$, with entropy integral $\mathcal J_\K$.

Then, we have the following regret result,
\begin{align*}
  \bigl[J(\pi^\star_{X,\K^\star})-c(\K^\star)\bigr]-\bigl[J(\hat\pi_{X,\hat\K})-c(\hat\K)\bigr]
  =O_P\Bigl(\frac{C_\rho(\max_\K\mathcal J_\K+\sqrt{\log N_b})}{\sqrt{n}}\Bigr)+o_P(n^{-1/2})
\end{align*}
See \labelcref{pf:budget-regret} for the proof.
\end{theorem}
This result tells us that our selection approach attains the best net value under budget $b$, which may improve upon the runtime optimum. 
In particular, since $\K=\emptyset$ is feasible, the theorem gives a lower bound to the estimated gap of $J(\hat\pi_{X,\hat\K})-c(\hat\K)\geq J(\pi^\star_X)+\bigl\{[J(\pi^\star_{X,\K^\star})-c(\K^\star)]-J(\pi^\star_X)\bigr\}-O_P(n^{-1/2})$. Hence, when the expression $[J(\pi^\star_{X,\K^\star})-c(\K^\star)]-J(\pi^\star_X)$ is positive, the selected policy improves on the runtime optimum. 

\section{Experiments}
\label{sec:experiments}

To examine the finite-sample performance of our methods, we conduct numerical experiments on fully synthetic data (Section~\ref{sec:synth-data}) and on semi-synthetic data from an Alzheimer's disease simulator (Section~\ref{sec:adcb-data}). In both, runtime confounding arises because variables $U_t$ that influence the logged policy are recorded in the historical data but withheld at deployment. In each experiment, we consider two tasks.

First, for policy learning and evaluation, we compare our proposed approach (RTC-DR), outlined in Algorithm~\ref{alg:rtc}, to two baselines, both of which apply the methods of \citet{KallusUehara2020DRL}: (1) a \textit{naive} DR method, which discards $U_t$ from the historical data and learns policies in $\Pi_X$, and (2) a full-information \textit{oracle}, which observes $U_t$ fully and learns policies in the oracle version $\Pi$ of $\Pi_X$, which adds $U_t$ under the same coefficient bound, so that the oracle differs from our approach only in having $U_t$ at deployment. In both experiments we use restricted policy classes whose coefficient constraints keep learned policies within the support of the logged policy and ensure bounded importance ratios, as required by our theory. Such behavior-constrained classes are standard in offline reinforcement learning \citep{fujimoto2019off,kumar2019stabilizing,wu2019behavior} and causal inference \citep{kennedy2019nonparametric}, and keep learned policies close to observed practice, which is desired in high-stakes settings \citep{gottesman2019guidelines}. We report the value of each learned policy, the RMSE and coverage of its estimated value, and its gap relative to the oracle.

Second, for budgeted runtime confounding, we apply the selection procedure of Section~\ref{sec:testing} while varying the cost per acquired covariate, and report the net value gained and the covariates acquired.

\subsection{Synthetic Data}
\label{sec:synth-data}

\paragraph{Setup.} We consider a binary action setting ($\A=\{0,1\}$) with 20 runtime-available covariates $X_t$ and 3 covariates $U_t$ that are available only in the historical data and unavailable at runtime. The parameter $c$ controls the severity of runtime confounding, with $c=0$ corresponding to no runtime confounding, and we consider two settings, a Markov and a fading-memory setting. We consider the restricted policy class and its oracle version
\begin{align*}
  \Pi_X&=\{\pi_\theta:\pi_\theta(1\mid X_t)=\ell(\theta_0+\theta_X^\top X_t),\ \|\theta\|_\infty\leq2\},\\
  \Pi&=\{\pi_\vartheta:\pi_\vartheta(1\mid X_t,U_t)=\ell(\vartheta_0+\vartheta_X^\top X_t+\vartheta_U^\top U_t),\ \|\vartheta\|_\infty\leq2\}.
\end{align*}
Full details on the data-generating mechanism are provided in Appendix~\ref{app:sim-details}, and additional results that vary the overall sample size $n$ and time periods $T$ are provided in Appendix~\ref{app:sim-results}.

\paragraph{Policy learning results.} Panel A of Figure~\ref{fig:learning-res} displays the results of our policy learning experiments. Under no runtime confounding ($c=0$), in both settings, naive DR and RTC-DR learn policies with similar value relative to the full-information oracle, although coverage remains below 95\% in some settings. As the degree of runtime confounding $c$ increases, naive DR learns policies with lower value than RTC-DR, with 95\% intervals that severely undercover the value of the learned policy. 
Importantly, for larger $c$, RTC-DR achieves regret similar to the full-information oracle whereas the naive DR method fails.
In Appendix~\ref{app:sim-results}, we report additional results varying $n$ and $T$, with analogous qualitative findings.

\paragraph{Budgeted runtime confounding results.} We additionally examine the performance of our proposed runtime confounder selection procedure as the cost to acquire each runtime covariate varies. Panels B and C of Figure~\ref{fig:learning-res} present the results, fixing $n=8000$ and $T=5$ under our fading-memory Markov setting. Using the split procedure of Section~\ref{sec:testing}, we learn policies for each of the 8 subsets of the 3 runtime confounders, and select the subset with the highest estimated value net of acquisition cost.  At low costs, two covariates are frequently selected, and the procedure selects fewer covariates as cost increases.  Value gains are positive over a large range of costs and gradually approach zero for larger costs, although costs slightly outweigh the value gained at some intermediate costs.

\begin{figure}
    \centering
    \begin{minipage}[t]{0.59\linewidth}
        \vspace{0pt}
        \centering
        \includegraphics[width=\linewidth]{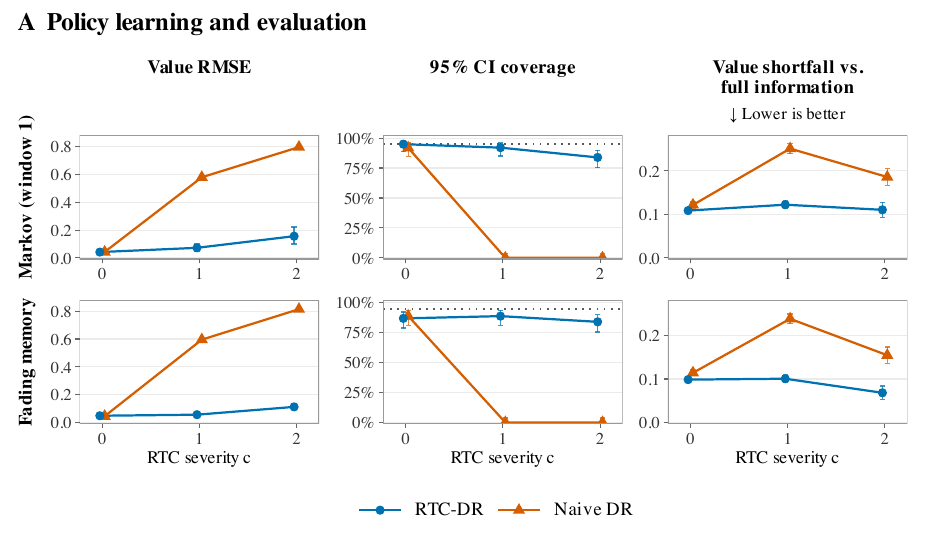}
    \end{minipage}\hfill
    \begin{minipage}[t]{0.39\linewidth}
        \vspace{0pt}
        \centering
        \includegraphics[width=\linewidth]{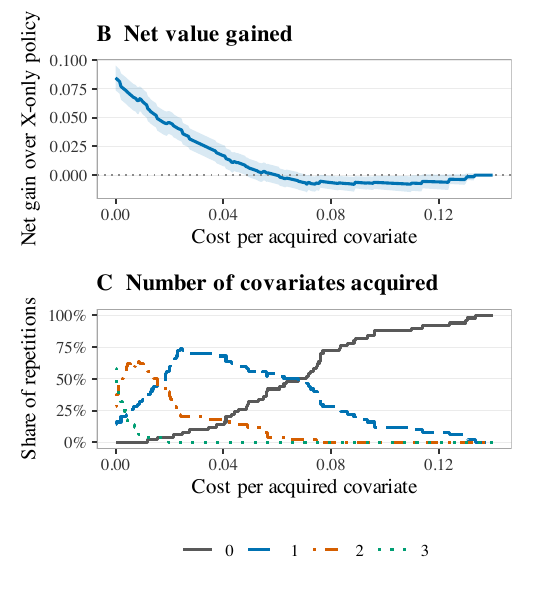}
    \end{minipage}
    \caption{Numerical experiment results, varying the severity of runtime confounding $c$ and Markov window, fixing $n=8000$ and $T=5$. \textbf{(A)} RMSE and coverage for learned policy value, and gap in learned policy value relative to the full-information learner. \textbf{(B)} Average net value gained from our covariate selection procedure, as function of cost per acquired covariate.  \textbf{(C)}  Mean share of covariates acquired from our selection procedure as a function of covariate costs.}
    \label{fig:learning-res}
\end{figure}

\subsection{ADCB Simulator}
\label{sec:adcb-data}

\paragraph{Setup.} We additionally assess the performance of our methods on semi-synthetic data from the Alzheimer's Disease Causal Estimation Benchmark (ADCB) simulator \citep{kinyanjui2022adcb}, under the simulator's Hernandez policy based on \citet{hernandez2010pharmacological}. The ADCB simulator generates longitudinal patient trajectories and cognitive outcomes under eight possible treatment actions according to this policy. We induce runtime confounding by withholding the Mini-Mental State Examination (MMSE) score, disease subtype and four biomarkers at deployment.

As in our synthetic experiments, we compare RTC-DR to the naive DR and full-information oracle approaches. Letting $\bar p_a$ be the mean probability of action $a$ under the logged policy, we consider the policy class $$\Pi_X=\left\{\pi_\theta:
\pi_\theta(a\mid H_t^X)
=\frac12\bar p_a+\frac12
\frac{\exp\{\theta_{a,0}+\theta_{a,X}^\top\widetilde X_t\}}
{\sum_{j=0}^7\exp\{\theta_{j,0}+\theta_{j,X}^\top\widetilde X_t\}},
\ \|\theta\|_\infty\leq2\right\},$$ while the full-information oracle learns policies in its oracle version $\Pi$, which additionally includes $\vartheta_{a,U}^\top U_t$. Along with policy learning, we consider a researcher choosing which runtime confounders to collect, implementing the procedure of Section~\ref{sec:testing} while varying the cost to acquire each covariate. For computational scalability, we restrict acquisition to at most one covariate, so the candidate subsets are $\K\in\{\emptyset,\{1\},\ldots,\{6\}\}$. Full details are provided in Appendix~\ref{app:adcb-details}.

\paragraph{Results.} Figure~\ref{fig:adcb-learning-results} and Table~\ref{tab:adcb-variable-gains} display the results of our analysis. RTC-DR and naive DR attain similar mean learned policy values, suggesting that correcting for runtime confounding during learning has limited impact here.
Nonetheless, runtime confounders still matter for decisions: having disease subtype available at runtime yields a positive mean value gain, while gains from MMSE and the four biomarkers are near zero, with intervals covering zero.
These results highlight the distinction between bias from ignoring runtime confounders, which our density ratios correct, and the value of having them at runtime, which no runtime policy can attain.
Both runtime methods choose actions based on the runtime variables only, marginalizing over the withheld variables.
In contrast, acquiring subtype allows actions to be chosen conditionally on this variable, recovering this value. At zero cost, our procedure acquires subtype in $58.7\%$ of repetitions and PTAU in only $1.1\%$, retaining the runtime policy otherwise.
Interestingly, the policy learned with subtype alone attains a slightly higher value than the full-information policy. This is likely because the full-information learner must also learn over the five withheld variables with near-zero value, which enlarges the policy class and adds estimation error without improving decisions. 

\begin{figure}
    \centering
    \includegraphics[width=\linewidth]{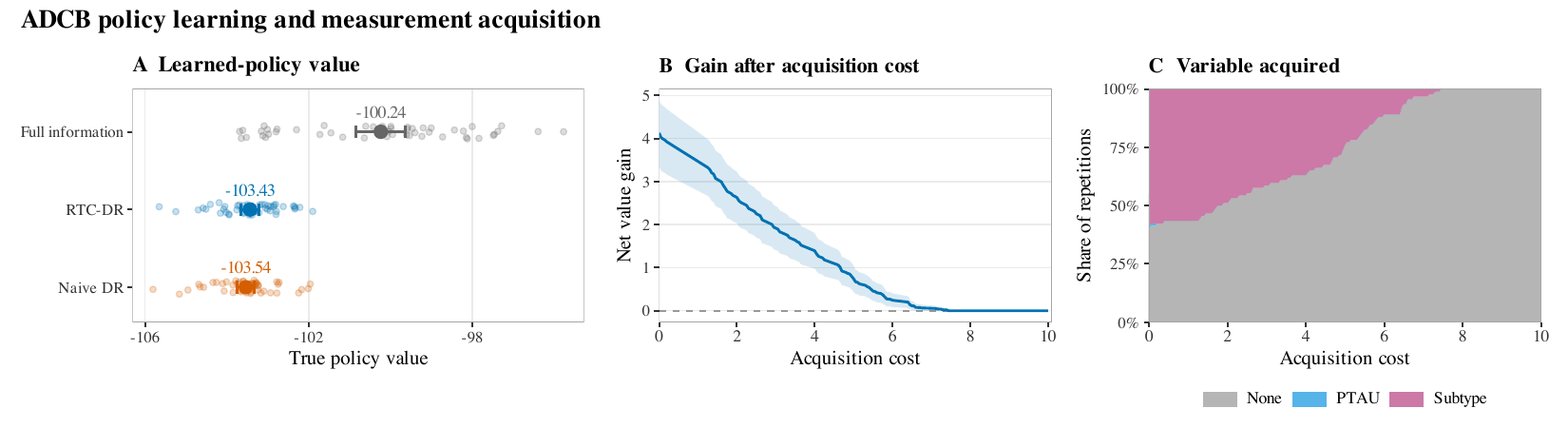}
    \caption{ADCB policy learning and budget results at $n=20000$ and $T=5$. \textbf{(A)} True learned policy values. \textbf{(B)} Gain in learned policy value net of acquisition cost as a function of acquisition cost.  \textbf{(C)} Acquisition frequencies as a function of acquisition cost. Bars and shaded regions denote 95\% Monte Carlo intervals.}
    \label{fig:adcb-learning-results}
\end{figure}

\begin{table}[htbp]
\centering
\begin{tabular}{lrrr}
\toprule
Variable & Mean value gain & 95\% MC interval & Acquired at zero cost \\
\midrule
MMSE    & $-0.10$ & $[-0.20,0.00]$ & $0.0\%$ \\
TAU     & $-0.14$ & $[-0.29,0.01]$ & $0.0\%$ \\
PTAU    & $-0.13$ & $[-0.26,0.01]$ & $1.1\%$ \\
FDG     & $ 0.04$ & $[-0.04,0.12]$ & $0.0\%$ \\
AV45    & $-0.03$ & $[-0.14,0.07]$ & $0.0\%$ \\
Subtype & $ 6.05$ & $[5.58,6.53]$  & $58.7\%$ \\
\bottomrule
\end{tabular}
\caption{ADCB value gains relative to the runtime policy when each variable is available at runtime, averaged over 92 repetitions, with 95\% Monte Carlo intervals. Acquisition frequencies apply the confidence check of Section~\ref{sec:testing} at zero cost; the runtime policy is retained in the remaining 40.2\% of repetitions.}
\label{tab:adcb-variable-gains}
\end{table}

\paragraph{Takeaways.} Across both experiments, we have two takeaways that support our theory. First, when runtime confounding biases naive learning, RTC-DR retains policy value and coverage as confounding grows, while naive DR loses value and its intervals severely under cover. Second, our budgeted selection procedure from \cref{sec:testing} identifies the runtime confounders that most improve decisions, acquiring when value exceeds cost, and maintaining the runtime policy otherwise.

\section{Discussion}
Under high-stakes or high-latency decision-making, deployment of standard methods can be biased due to variables present in training logs not being present during implementation. 
We introduce policy learning and semiparametric efficient off-policy evaluation methods for reinforcement learning under runtime confounding to overcome this gap. 
We additionally provide diagnostics that quantify the value lost from runtime confounding, and a budgeted procedure for selecting which runtime confounders to collect to reduce it. 
Together, our methods provide a concrete pipeline to take historical logs to decisions for future studies and deployment. Extensions to settings with decaying overlap are left for future work.

\section*{AI Disclosure Statement}
We used Anthropic's Claude and OpenAI's Codex as  general-purpose assistants for editing and drafting text, verifying derivations, and coding assistance for our numerical experiments. All outputs were reviewed by the authors and all results and experiments were designed, implemented, and verified by the authors, who take full responsibility for the content of this submission.

\acks{This work was supported by National Institutes of Health (NIH) grant T32HL139426, and the National Science Foundation Grant No.~2306556. We thank Rachel Nethery for her amazing support, feedback, and bringing people together.}

\clearpage 

\bibliography{sources} 

@inproceedings{coston2020counterfactual,
  title={Counterfactual Predictions under Runtime Confounding},
  author={Coston, Amanda and Kennedy, Edward and Chouldechova, Alexandra},
  booktitle={Advances in Neural Information Processing Systems},
  volume={33},
  pages={4150--4162},
  year={2020}
}

@article{KallusUehara2020DRL,
  author  = {Nathan Kallus and Masatoshi Uehara},
  title   = {Double Reinforcement Learning for Efficient Off-Policy Evaluation in Markov Decision Processes},
  journal = {Journal of Machine Learning Research},
  volume  = {21},
  number  = {167},
  pages   = {1--63},
  year    = {2020},
  url     = {https://jmlr.org/papers/v21/19-827.html}
}

@InProceedings{barnatchez2026conformal,
  title = 	 {Debiased Machine Learning for Conformal Prediction of Counterfactual Outcomes Under Runtime Confounding},
  author =       {Barnatchez, Keith and Josey, Kevin P. and Nethery, Rachel C. and Parmigiani, Giovanni},
  booktitle = 	 {Proceedings of the Fifth Conference on Causal Learning and Reasoning},
  pages = 	 {1345--1379},
  year = 	 {2026},
  volume = 	 {323},
  series = 	 {Proceedings of Machine Learning Research}
}

@article{bennett2021proximal,
  title={Proximal Reinforcement Learning: Efficient Off-Policy Evaluation in Partially Observed {M}arkov Decision Processes},
  author={Bennett, Andrew and Kallus, Nathan},
  journal={Operations Research},
  volume={72},
  number={3},
  pages={1071--1086},
  year={2024}
}

@inproceedings{uehara2023future,
  title={Future-Dependent Value-Based Off-Policy Evaluation in {POMDP}s},
  author={Uehara, Masatoshi and Kiyohara, Haruka and Bennett, Andrew and Chernozhukov, Victor and Jiang, Nan and Kallus, Nathan and Shi, Chengchun and Sun, Wen},
  booktitle={Advances in Neural Information Processing Systems},
  volume={36},
  year={2023}
}

@article{athey2021policy,
  title={Policy Learning with Observational Data},
  author={Athey, Susan and Wager, Stefan},
  journal={Econometrica},
  volume={89},
  number={1},
  pages={133--161},
  year={2021}
}

@article{chernozhukov2018double,
  title={Double/debiased machine learning for treatment and structural parameters},
  author={Chernozhukov, Victor and Chetverikov, Denis and Demirer, Mert and Duflo, Esther and Hansen, Christian and Newey, Whitney and Robins, James},
  journal={The Econometrics Journal},
  volume={21},
  number={1},
  pages={C1--C68},
  year={2018}
}

@inproceedings{jiang2016doubly,
  title={Doubly Robust Off-policy Value Evaluation for Reinforcement Learning},
  author={Jiang, Nan and Li, Lihong},
  booktitle={Proceedings of the 33rd International Conference on Machine Learning},
  series={Proceedings of Machine Learning Research},
  volume={48},
  pages={652--661},
  year={2016}
}

@inproceedings{thomas2016data,
  title={Data-Efficient Off-Policy Policy Evaluation for Reinforcement Learning},
  author={Thomas, Philip and Brunskill, Emma},
  booktitle={Proceedings of the 33rd International Conference on Machine Learning},
  series={Proceedings of Machine Learning Research},
  volume={48},
  pages={2139--2148},
  year={2016}
}

@article{zhang2013robust,
  title={Robust estimation of optimal dynamic treatment regimes for sequential treatment decisions},
  author={Zhang, Baqun and Tsiatis, Anastasios A. and Laber, Eric B. and Davidian, Marie},
  journal={Biometrika},
  volume={100},
  number={3},
  pages={681--694},
  year={2013}
}

@article{luedtke2016statistical,
  title={Statistical inference for the mean outcome under a possibly non-unique optimal treatment strategy},
  author={Luedtke, Alexander R. and van der Laan, Mark J.},
  journal={The Annals of Statistics},
  volume={44},
  number={2},
  pages={713--742},
  year={2016}
}

@inproceedings{bareinboim2015bandits,
  title={Bandits with Unobserved Confounders: A Causal Approach},
  author={Bareinboim, Elias and Forney, Andrew and Pearl, Judea},
  booktitle={Advances in Neural Information Processing Systems},
  volume={28},
  year={2015}
}

@inproceedings{zhang2019near,
  title={Near-Optimal Reinforcement Learning in Dynamic Treatment Regimes},
  author={Zhang, Junzhe and Bareinboim, Elias},
  booktitle={Advances in Neural Information Processing Systems},
  volume={32},
  year={2019}
}

@inproceedings{kallus2018confounding,
  title={Confounding-Robust Policy Improvement},
  author={Kallus, Nathan and Zhou, Angela},
  booktitle={Advances in Neural Information Processing Systems},
  volume={31},
  year={2018}
}

@inproceedings{namkoong2020off,
  title={Off-policy Policy Evaluation For Sequential Decisions Under Unobserved Confounding},
  author={Namkoong, Hongseok and Keramati, Ramtin and Yadlowsky, Steve and Brunskill, Emma},
  booktitle={Advances in Neural Information Processing Systems},
  volume={33},
  year={2020}
}

@inproceedings{tennenholtz2020off,
  title={Off-Policy Evaluation in Partially Observable Environments},
  author={Tennenholtz, Guy and Shalit, Uri and Mannor, Shie},
  booktitle={Proceedings of the AAAI Conference on Artificial Intelligence},
  volume={34},
  pages={10276--10283},
  year={2020}
}

@article{van2015targeted,
  title={Targeted learning of the mean outcome under an optimal dynamic treatment rule},
  author={van der Laan, Mark J and Luedtke, Alexander R},
  journal={Journal of causal inference},
  volume={3},
  number={1},
  pages={61--95},
  year={2015},
  publisher={De Gruyter}
}

@article{kennedy2024semiparametric,
  title={Semiparametric doubly robust targeted double machine learning: a review},
  author={Kennedy, Edward H},
  journal={Handbook of statistical methods for precision medicine},
  pages={207--236},
  year={2024},
  publisher={Chapman and Hall/CRC}
}

@inproceedings{kinyanjui2022adcb,
  title={{ADCB}: An {Alzheimer}'s Disease Simulator for Benchmarking Observational Estimators of Causal Effects},
  author={Kinyanjui, Newton Mwai and Johansson, Fredrik D},
  booktitle={Conference on Health, Inference, and Learning},
  pages={103--118},
  year={2022},
  organization={PMLR}
}

@article{luedtke2016super,
  title   = {Super-Learning of an Optimal Dynamic Treatment Rule},
  author  = {Luedtke, Alexander R. and van der Laan, Mark J.},
  journal = {The International Journal of Biostatistics},
  year    = {2016},
  volume  = {12},
  number  = {1},
  pages   = {305--332},
  doi     = {10.1515/ijb-2015-0052}
}

@article{murphy2001marginal,
  title   = {Marginal Mean Models for Dynamic Regimes},
  author  = {Murphy, Susan A. and van der Laan, Mark J.
             and Robins, James M.
             and {Conduct Problems Prevention Research Group}},
  journal = {Journal of the American Statistical Association},
  year    = {2001},
  volume  = {96},
  number  = {456},
  pages   = {1410--1423},
  doi     = {10.1198/016214501753382327}
}

@article{sakaguchi2025estimation,
  title   = {Estimation of Optimal Dynamic Treatment Assignment
             Rules under Policy Constraints},
  author  = {Sakaguchi, Shosei},
  journal = {Quantitative Economics},
  year    = {2025},
  volume  = {16},
  number  = {3},
  pages   = {981--1022},
  doi     = {10.3982/QE2288}
}

@article{hernandez2010pharmacological,
  title   = {Pharmacological Treatment of {Alzheimer's} Disease:
             Effect of Race and Demographic Variables},
  author  = {Hernandez, Santiago and McClendon, McKee J. and
             Zhou, Xiao-Hua Andrew and Sachs, Michael and
             Lerner, Alan J.},
  journal = {Journal of Alzheimer's Disease},
  year    = {2010},
  volume  = {19},
  number  = {2},
  pages   = {665--672},
  doi     = {10.3233/JAD-2010-1269}
}

@inproceedings{fujimoto2019off,
  title={Off-Policy Deep Reinforcement Learning without Exploration},
  author={Fujimoto, Scott and Meger, David and Precup, Doina},
  booktitle={Proceedings of the 36th International Conference on Machine Learning},
  series={Proceedings of Machine Learning Research},
  volume={97},
  pages={2052--2062},
  year={2019}
}

@inproceedings{kumar2019stabilizing,
  title={Stabilizing Off-Policy {Q}-Learning via Bootstrapping Error Reduction},
  author={Kumar, Aviral and Fu, Justin and Soh, Matthew and Tucker, George and Levine, Sergey},
  booktitle={Advances in Neural Information Processing Systems},
  volume={32},
  year={2019}
}

@article{wu2019behavior,
  title={Behavior Regularized Offline Reinforcement Learning},
  author={Wu, Yifan and Tucker, George and Nachum, Ofir},
  journal={arXiv preprint arXiv:1911.11361},
  year={2019}
}

@article{kennedy2019nonparametric,
  title={Nonparametric Causal Effects Based on Incremental Propensity Score Interventions},
  author={Kennedy, Edward H.},
  journal={Journal of the American Statistical Association},
  volume={114},
  number={526},
  pages={645--656},
  year={2019}
}

@article{gottesman2019guidelines,
  title={Guidelines for Reinforcement Learning in Healthcare},
  author={Gottesman, Omer and Johansson, Fredrik and Komorowski, Matthieu and Faisal, Aldo and Sontag, David and Doshi-Velez, Finale and Celi, Leo Anthony},
  journal={Nature Medicine},
  volume={25},
  number={1},
  pages={16--18},
  year={2019}
}

@article{luckett2020estimating,
  title={Estimating dynamic treatment regimes in mobile health using v-learning},
  author={Luckett, Daniel J and Laber, Eric B and Kahkoska, Anna R and Maahs, David M and Mayer-Davis, Elizabeth and Kosorok, Michael R},
  journal={Journal of the american statistical association},
  year={2020},
  publisher={Taylor \& Francis}
}

\clearpage 

\appendix

\section{Proofs}
\subsection{\texorpdfstring{Proof of \cref{thm:EIF}}{}}\label{pf:EIF}

\begin{proof}
    In general, note the factorization given by
    $$ \pr(h_T) = \pr(s_1) \prod_t \pi_b(a_t\mid h_t) \pr_t(r_t \mid h_t, a_t)\pr_t(s_{t+1} \mid r_t, h_t, a_t )$$

    The EIF will be derived using the point mass contamination method, by perturbations on each of these densities, except $\pi_b$, which is assumed to be logged. Let $Z=(S,A,R)$, $S=(X,U)$, we consider perturbations of the form $\pr_\varepsilon = (1-\varepsilon)\pr + \varepsilon \delta_{\tilde z}$

    \paragraph{T=1} We start with the case when $T=1$. Define $m(S)=\sum_a\pi(a\mid X)Q(S,a)$, $Q(s,a)=\E[R\mid s,a]$.
    
    Our target parameter is $$J(\pr)=\int \sum_a \pi(a \mid x) Q(s,a) f(s) ds$$ 

    Perturb along each population object: $Q(s,a), f(s)$. Then, by the chain rule, we have
    \begin{align*}
        \frac{\mathrm d}{\mathrm d\varepsilon}J(\pr_\varepsilon)\Big|_0 = \sum_a \pi(a \mid \tilde x) Q(\tilde s, a) - J + \int \sum_a \pi(a \mid x) \big(\frac{\mathrm d}{\mathrm d\varepsilon} Q(s,a) |_{\varepsilon=0} \big) f(s) ds
    \end{align*}    

    Noting the usual calculation,
    $$ \frac{\mathrm d}{\mathrm d\varepsilon} \E_{\pr_\varepsilon}[Y \mid Z=z] |_{\varepsilon=0} = \I{z=\tilde z}/f(z)(\tilde y - \E[Y\mid Z=z]) $$

    This reduces to
    \begin{align*}
        \frac{\mathrm d}{\mathrm d\varepsilon}J(\pr_\varepsilon)\Big|_0 &= \big( \sum_a \pi(a \mid \tilde x) Q(\tilde s, a) - J\big) + \pi(\tilde a \mid \tilde x) /f(\tilde s, \tilde a) (\tilde r - Q(\tilde s, \tilde a)) f(\tilde s) \\
        &= \big( \sum_a \pi(a \mid \tilde x) Q(\tilde s, a) - J \big) + \pi(\tilde a \mid \tilde x) /\pi_b(\tilde a \mid \tilde s) (\tilde r - Q(\tilde s, \tilde a))
    \end{align*}    

    Hence, the EIF at $T=1$ reduces to $\varphi=m(S)-J+\lambda^{\RTC}_1(R-Q(S,A))$, for importance weight at one timepoint given by $\lambda^{\RTC}_1=\pi_1(A_1\mid X_1)/\pi_b(A_1\mid H_1)$.

    \paragraph{General T}
    Now, we wish to extend this logic to general $T$. The difference is that $V_t$ is a recursive object: 
    $$ V_t^\pi(H_t) = \sum_a \pi_t(a \mid \HX_t) Q_t^\pi(H_t,a) = \sum_a \pi_t(a \mid \HX_t) \E[R_t + \gamma V_{t+1}^\pi(H_{t+1}) \mid H_t,A_t=a] $$

    In the same manner as above, we have 
    \begin{align*}
        \frac{\mathrm d}{\mathrm d\varepsilon}J(\pr_\varepsilon)\Big|_0 &= V_1^\pi(\tilde H_1)- J + \E[\big(\frac{\mathrm d}{\mathrm d\varepsilon}  V_1^\pi(H_1)|_{\varepsilon=0} \big)] 
    \end{align*}
    where we have used that $\E[V_1^\pi(H_1)]=\E_\pi[\sum_t \gamma^{t-1} R_t]=J(\pi)$

    For convenience, define the following time $t$ IS-weighted EIF of $V$ as
    $$ \Gamma_t = \E[\rho^{\RTC}_{1:t-1} \frac{\mathrm d}{\mathrm d\varepsilon} V_t^\pi(H_t)|_{\varepsilon=0}]$$
    Setting the base case for our weights to 1, $\rho^{\RTC}_{1:0}=1$, gives $\Gamma_1=\E[\big(\frac{\mathrm d}{\mathrm d\varepsilon}  V_1^\pi(H_1)|_{\varepsilon=0} \big)] $ as our object as interest. 
    
    Since the value function is a function of the future timepoints (e.g. $\Gamma_1 = \E[\frac{\mathrm d}{\mathrm d\varepsilon} V_1^\pi(H_1)|_{\varepsilon=0}] $ for $V_1$ being a function of $V_2$), we require unrolling this sum and understanding $\Gamma_t$. 

    Borrowing the key calculation from above on $Q$, we perturb the distributions, and calculate the EIF on $Q_{t}$:
    \begin{align*}
        \frac{\mathrm d}{\mathrm d\varepsilon} Q_{t}^\pi(H_{t}, a)|_{\varepsilon=0} &= \I{ H_t = \tilde h_t, a = \tilde a_t}/f(H_t,a) (\tilde r_t + \gamma V_{t+1}^\pi(\tilde H_{t+1})-Q_t^\pi(H_t,a)) \\
        &\qt + \gamma \E[\frac{\mathrm d}{\mathrm d\varepsilon}V_{t+1}^\pi(H_{t+1})|_{\varepsilon=0} \mid H_t, A_t=a]
    \end{align*}

    Plugging this into $\Gamma_t$ (the directional derivative for $V_t^\pi$ along $\pr_\varepsilon$), using $f(\tilde h_t,\tilde a_t)=f(\tilde h_t)\pi_b(\tilde a_t\mid \tilde h_t)$ in the first term and $\E[\sum_a \pi_t(a\mid \HX_t) g(H_t,a)]=\E[\lambda^{\RTC}_t g(H_t,A_t)]$ in the second, gives
    \begin{align*}
        \Gamma_t &= \rho^{\RTC}_{1:t-1} \frac{\pi_t(\tilde a_t \mid \tilde h^X_t)}{\pi_b(\tilde a_t \mid \tilde h_t)} (\tilde r_t + \gamma V_{t+1}^\pi(\tilde H_{t+1})-Q_t^\pi(\tilde h_t,\tilde a_t)) + \gamma \E[\rho^{\RTC}_{1:t-1}\lambda^{\RTC}_t \frac{\mathrm d}{\mathrm d\varepsilon}V_{t+1}^\pi(H_{t+1})|_{\varepsilon=0}] \\
        &= \rho^{\RTC}_{1:t} (\tilde r_t + \gamma V_{t+1}^\pi(\tilde H_{t+1})-Q_t^\pi(\tilde h_t,\tilde a_t)) + \gamma \Gamma_{t+1}
    \end{align*}
    where $\rho^{\RTC}_{1:t}$ in the last line is evaluated at the point mass.
    
    Unroll this from $t=1$, with $\Gamma_{T+1}=0$,
    $$ \Gamma_1 = \sum_{t=1}^T \gamma^{t-1}\rho^{\RTC}_{1:t}(\tilde r_t + \gamma V_{t+1}^\pi(\tilde H_{t+1})-Q_t^\pi(\tilde h_t,\tilde a_t)) $$

    Hence, the EIF form follows.

    Note, if $\pi_b$ is unknown and we need to account for local perturbations in the derivation, just as in the iid case, our score is orthogonal yielding the same EIF as above.
\end{proof}

\subsection{\texorpdfstring{Proof of \cref{thm:regret}}{}}\label{pf:regret}

\textit{Step 1: Regret decomposition.} Adding and subtracting $J(\pi^\star_X)$, the value at the optimal runtime-confounded policy, we can decompose the regret as
\begin{equation}
  J(\pi^\star)-J(\hat\pi)
  =\underbrace{J(\pi^\star)-J(\pi^\star_X)}_{\Delta_{\RTC}} + 
  \underbrace{J(\pi^\star_X)-J(\hat\pi)}_{\text{estimation}}
  \label{eq:three}
\end{equation}

$\Delta_{\RTC}$ is irreducible, representing the gap between the optimal policy and optimal runtime policy. It remains to bound the estimation term, which is the only term depending on the data.

\paragraph*{Step 1: Value-error decomposition.}
We can bound the estimation term as follows:
\begin{align}
  J(\pi^\star_\Pi)-J(\hat\pi)
  &=\bigl[J(\pi^\star_\Pi)-\hat J(\pi^\star_\Pi)\bigr]+\bigl[\hat J(\pi^\star_\Pi)-\hat J(\hat\pi)\bigr]+\bigl[\hat J(\hat\pi)-J(\hat\pi)\bigr]\notag\\
  &\leq \bigl[\hat J-J\bigr](\hat\pi)-\bigl[\hat J-J\bigr](\pi^\star_\Pi)\notag\\
  &=(\pr_n-\E)\bigl(\varphi^{\hat\pi}-\varphi^{\pi^\star_\Pi}\bigr)
   +\mathrm{Rem}(\hat\eta;\hat\pi)-\mathrm{Rem}(\hat\eta;\pi^\star_\Pi)\notag \\
   &\qt +(\pr_n-\E)\bigl[(\widehat\varphi-\varphi)^{\hat\pi}-(\widehat\varphi-\varphi)^{\pi^\star_\Pi}\bigr] \notag\\
  &\leq 2\sup_{\pi\in\Pi}\bigl|(\pr_n-\E)\varphi^\pi\bigr|
   +2\sup_{\pi\in\Pi}\bigl|(\pr_n-\E)(\widehat\varphi^\pi-\varphi^\pi)\bigr|
   +\bigl|\mathrm{Rem}(\hat\eta;\hat\pi)\bigr|+\bigl|\mathrm{Rem}(\hat\eta;\pi^\star_\Pi)\bigr|
  \label{eq:split}
\end{align}
The first inequality follows by adding and subtracting $\hat{J}(\pi_\Pi^*)$ and $\hat{J}(\hat{\pi})$, and applying the basic inequality. The second equality follows plugging in \eqref{eq:expand} and expanding. 
The final inequality follows from the triangle inequality and realizability.

\paragraph*{Step 2: Empirical process control, on supremum terms.}

\subparagraph{Term $2\sup_{\pi\in\Pi}\bigl|(\pr_n-\E)\varphi^\pi\bigr|$:}
By \labelcref{ass:overlap}, $|V^\pi_1|,|J(\pi)|\leq C_Q$, $\rho^{\RTC}_{1:t}\leq\eta_g^{-t}$ and $|\delta^\pi_t|\leq R_{\max}+(1+\gamma)C_Q$, and the same bounds hold for the estimated versions. Hence $\Phi_\Pi$ and $\widehat\Phi^{(k)}_\Pi$ admit an envelope $C_\rho=2C_Q+(R_{\max}+2C_Q)\sum_{t=1}^T\gamma^{t-1}\eta_g^{-t}$. Then, by symmetrization and Dudley's entropy bound under \labelcref{ass:complex},
\[
  \E\sup_{\pi\in\Pi}\bigl|(\pr_n-\E)\varphi^\pi\bigr|\lesssim\frac{C_\rho\mathcal J}{\sqrt n}
\]

\subparagraph{Term $2\sup_{\pi\in\Pi}\bigl|(\pr_n-\E)(\widehat\varphi^\pi-\varphi^\pi)\bigr|$:}
Cross-fitting splits the term over folds of size $n_k=n/K$. Let the function class implied by this process be given by $\mathcal G_k=\{\varphi^\pi(\cdot;\hat\eta^{(-k)})-\varphi^\pi(\cdot;\eta):\pi\in\Pi\}$. Then,
\[
  \sup_{\pi\in\Pi}\bigl|(\pr_n-\E)(\widehat\varphi^\pi-\varphi^\pi)\bigr|
  \le\sum_{k=1}^K\frac{n_k}{n}\sup_{g\in\mathcal G_k}\bigl|(\pr_{n_k}-\E)g\bigr|
\]
Conditional on the training folds $\mathcal I_{-k}$, the nuisance estimates are fixed, so $\mathcal G_k$ is a fixed class of differences between elements of $\widehat\Phi^{(k)}_\Pi$ and $\Phi_\Pi$, with envelope $2C_\rho$. Its entropy is at most the sum of the entropies of these two classes, so $\mathcal J(\mathcal G_k)\leq 2\mathcal J$ by \labelcref{ass:complex}.
Dudley's integral on the evaluation fold gives $\E[\sup_{g\in\mathcal G_k}|(\pr_{n_k}-\E)g|\mid\mathcal I_{-k}]\lesssim C_\rho\mathcal J/\sqrt{n_k}$. Since $K$ is fixed, Markov's inequality gives $2\sup_{\pi\in\Pi}\bigl|(\pr_n-\E)(\widehat\varphi^\pi-\varphi^\pi)\bigr|=O_P(C_\rho\mathcal J/\sqrt n)$.

\paragraph*{Step 3: Second-order remainder (R)}
By \eqref{eq:exact}, applied fold by fold with $\hat\eta=\hat\eta^{(-k)}$, for any $\pi\in\Pi_X$,
\begin{align*}
  \bigl|\mathrm{Rem}(\hat\eta;\pi)\bigr|
  &\leq\sum_{t=1}^T\gamma^{t-1}\E\Bigl[\hat\rho^{\RTC}_{1:t-1}\frac{1}{\min_a\hat g_t(a\mid H_t)}\sum_a\pi_t(a\mid\HX_t)|\Delta g_t(a\mid H_t)|\cdot|\Delta q^\pi_t(H_t,a)|\Bigr]\\
  &\leq\sum_{t=1}^T\gamma^{t-1}\eta_g^{-t}\E\Bigl[\max_a|\Delta g_t(a\mid H_t)|\cdot\max_a|\Delta q^\pi_t(H_t,a)|\Bigr]\\
  &\leq\sum_{t=1}^T\gamma^{t-1}\eta_g^{-t}\Bigl\lVert\max_a|\Delta g_t(a\mid\cdot)|\Bigr\rVert_2\Bigl\lVert\max_a|\Delta q^\pi_t(\cdot,a)|\Bigr\rVert_2
\end{align*}
where the second inequality uses $\hat\rho^{\RTC}_{1:t-1}\leq\eta_g^{-(t-1)}$ from \labelcref{ass:overlap}, and the third follows from Cauchy--Schwarz. 
\labelcref{ass:product} then shows $|\mathrm{Rem}(\hat\eta;\hat\pi)|+|\mathrm{Rem}(\hat\eta;\pi^\star_X)|=o_P(n^{-1/2})$.

\subsection{\texorpdfstring{Proof of \cref{thm:clt}}{}}\label{pf:clt}
\begin{proof}
    
\paragraph*{Decomposition.}
Our decomposition yielded
\[
  \hat J(\pi)-J(\pi)
  =\underbrace{(\pr_n-\E)\varphi^\pi}_{\text{(E)}}
  +\underbrace{\mathrm{Rem}(\hat\eta;\pi)}_{\text{(R)}}
  +\underbrace{(\pr_n-\E)(\widehat\varphi^\pi-\varphi^\pi)}_{\text{(C)}}
\]
We show the result by demonstrating $\sqrt n (\pr_n-\E)\varphi^\pi \rightsquigarrow\mathcal N(0,\sigma^2_{\RTC})$ and that (R) and (C) are $o_P(n^{-1/2})$.

\paragraph*{Term (E) and the variance.}
By \labelcref{ass:overlap}, $\varphi^\pi$ is bounded by $C_\rho$, and $\E[\varphi^\pi]=0$. Since trajectories are i.i.d., the central limit theorem gives $\sqrt n (\pr_n-\E)\varphi^\pi\rightsquigarrow\mathcal N(0,\V[\varphi^\pi])$. 

We are explicit about the form of $\V[\varphi^\pi]$. Write $\varphi^\pi=\xi_0+\sum_{t=1}^T\xi_t$, where
\[
  \xi_0=V^\pi_1(H_1)-J(\pi),\qquad \xi_t=\gamma^{t-1}\rho^{\RTC}_{1:t}\delta^\pi_t
\]
for Bellman residual $\delta^\pi_t=R_t+\gamma V^\pi_{t+1}(H_{t+1})-Q^\pi_t(H_t,A_t)$.
Expanding the variance of the sum,
\begin{equation}
  \V[\varphi^\pi]=\V[\xi_0]+\sum_{t=1}^T\V[\xi_t]+2\sum_{t=1}^T\mathrm{Cov}(\xi_0,\xi_t)+2\sum_{1\le s<t\le T}\mathrm{Cov}(\xi_s,\xi_t)
  \label{eq:varexp}
\end{equation}

\textbf{Mean Zero.} By the tower property and the Bellman equation  $\E[\delta^\pi_t\mid H_t,A_t]=\E[R_t+\gamma V^\pi_{t+1}(H_{t+1})\mid H_t,A_t]-Q^\pi_t(H_t,A_t)=0$,

\[
  \E[\xi_t]=\E\bigl[\gamma^{t-1}\rho^{\RTC}_{1:t}\E[\delta^\pi_t\mid H_t,A_t]\bigr]=0,\qquad \E[\xi_0]=\E[V^\pi_1(H_1)]-J(\pi)=0
\]

\textbf{Covariances with $\xi_0$.} Conditioning on $(H_t,A_t)$ and pulling out the measurable factors $\xi_0$ and $\gamma^{t-1}\rho^{\RTC}_{1:t}$,
\[
  \mathrm{Cov}(\xi_0,\xi_t)=\E[\xi_0\xi_t]=\E\bigl[\xi_0\gamma^{t-1}\rho^{\RTC}_{1:t}\E[\delta^\pi_t\mid H_t,A_t]\bigr]=0
\]

\textbf{Covariances across periods.} For $s<t$, $\xi_s$ is $\sigma(H_t,A_t)$-measurable, so
\[
  \mathrm{Cov}(\xi_s,\xi_t)=\E[\xi_s\xi_t]=\E\bigl[\xi_s\gamma^{t-1}\rho^{\RTC}_{1:t}\E[\delta^\pi_t\mid H_t,A_t]\bigr]=0
\]

\textbf{Variances.} $\E[(\delta^\pi_t)^2\mid H_t,A_t]=\V[\delta^\pi_t\mid H_t,A_t]$, so
\[
  \V[\xi_t]=\E[\xi_t^2]=\gamma^{2(t-1)}\E\bigl[(\rho^{\RTC}_{1:t})^2\E[(\delta^\pi_t)^2\mid H_t,A_t]\bigr]=\gamma^{2(t-1)}\E\bigl[(\rho^{\RTC}_{1:t})^2\V[\delta^\pi_t\mid H_t,A_t]\bigr]
\]
and $\V[\xi_0]=\V[V^\pi_1(H_1)]$. 

\textbf{Combining.}
Substituting into \eqref{eq:varexp}, we find
\[
  \V[\varphi^\pi]=\V\bigl[V^\pi_1(H_1)\bigr]+\sum_{t=1}^T\gamma^{2(t-1)}\E\bigl[(\rho^{\RTC}_{1:t})^2\V[\delta^\pi_t\mid H_t,A_t]\bigr]=\sigma^2_{\RTC}
\]

Hence
\[
    \V[\varphi^\pi]=\V\bigl[V^\pi_1(H_1)\bigr]+\sum_{t=1}^T\E[\xi_t^2]=\sigma^2_{\RTC}
\]
which is the efficiency bound since $\varphi^\pi$ is the EIF (Theorem~\ref{thm:EIF}).

\paragraph*{Term (R).} See \cref{thm:clt} for the proof.

By the bound on the remainder in the proof of Theorem~\ref{thm:regret}, at the fixed $\pi$,
\[
  \bigl|\mathrm{Rem}(\hat\eta;\pi)\bigr|\le\sum_{t=1}^T\gamma^{t-1}\eta_g^{-t}\lVert\Delta g_t\rVert_2\lVert\Delta q^\pi_t\rVert_2=o_P(n^{-1/2})
\]
by \labelcref{ass:product}.

\paragraph*{Term (C).}
Fix a fold $k$ with $n_k=n/K$ units. Conditional on the training folds $\mathcal I_{-k}$, $\widehat\varphi^{\pi,(-k)}-\varphi^\pi$ is a fixed function so
\[
  \E\bigl[(\pr_{n_k}-\E)(\widehat\varphi^{\pi,(-k)}-\varphi^\pi)\mid\mathcal I_{-k}\bigr]=0,
  \qquad
  \V\bigl[(\pr_{n_k}-\E)(\widehat\varphi^{\pi,(-k)}-\varphi^\pi)\mid\mathcal I_{-k}\bigr]\le\frac{\lVert\widehat\varphi^{\pi,(-k)}-\varphi^\pi\rVert_2^2}{n_k}
\]
By Chebyshev's inequality, this term is $O_P(\lVert\widehat\varphi^{\pi,(-k)}-\varphi^\pi\rVert_2/\sqrt n)$. The difference can be decomposed as
\[
  \widehat\varphi^\pi-\varphi^\pi=(\widehat V^\pi_1-V^\pi_1)+\sum_{t=1}^T\gamma^{t-1}\bigl[(\hat\rho^{\RTC}_{1:t}-\rho^{\RTC}_{1:t})\hat\delta^\pi_t+\rho^{\RTC}_{1:t}(\hat\delta^\pi_t-\delta^\pi_t)\bigr]
\]
By \labelcref{ass:overlap}, $|\widehat V^\pi_1-V^\pi_1|\le\max_a|\Delta q^\pi_1|$, $|\hat\delta^\pi_t|\le R_{\max}+(1+\gamma)C_Q$. Also, $|\hat\delta^\pi_t-\delta^\pi_t|\le\max_a|\Delta q^\pi_t|+\gamma\max_a|\Delta q^\pi_{t+1}|$, and $\rho^{\RTC}_{1:t}$ is at most $\eta_g^{-1}$. Combining, we can bound the $\rho$ difference by
\[
  |\hat\rho^{\RTC}_{1:t}-\rho^{\RTC}_{1:t}|\le\eta_g^{-(t-1)}\sum_{s\le t}\Bigl|\frac{\pi_s}{\hat g_s}-\frac{\pi_s}{\pi_b}\Bigr|\le\eta_g^{-(t+1)}\sum_{s\le t}\max_a|\Delta g_s|
\]
Hence by \labelcref{ass:consist}, $\lVert\widehat\varphi^{\pi,(-k)}-\varphi^\pi\rVert_2\lesssim\sum_{t=1}^T(\lVert\Delta g_t\rVert_2+\lVert\Delta q^\pi_t\rVert_2)=o_P(1)$.

\paragraph*{Consistency of our Variance Estimator.}
Since $|\widehat\varphi^\pi|,|\varphi^\pi|\le C_\rho$,
\[
  \bigl|\pr_n(\widehat\varphi^\pi)^2-\pr_n(\varphi^\pi)^2\bigr|=\bigl|\pr_n[(\widehat\varphi^\pi-\varphi^\pi)(\widehat\varphi^\pi+\varphi^\pi)]\bigr|\le2C_\rho\pr_n|\widehat\varphi^\pi-\varphi^\pi|
\]
With cross-fitting, $\pr_n|\widehat\varphi^\pi-\varphi^\pi|=\sum_{k=1}^K\frac{n_k}{n}\pr_{n_k}|\widehat\varphi^{\pi,(-k)}-\varphi^\pi|$. Conditional on $\mathcal I_{-k}$, $\widehat\varphi^{\pi,(-k)}-\varphi^\pi$ is a fixed function evaluated on the independent fold $k$, so by Cauchy--Schwarz
\[
  \E\bigl[\pr_{n_k}|\widehat\varphi^{\pi,(-k)}-\varphi^\pi|\bigm|\mathcal I_{-k}\bigr]=\E\bigl[|\widehat\varphi^{\pi,(-k)}-\varphi^\pi|\bigm|\mathcal I_{-k}\bigr]\le\lVert\widehat\varphi^{\pi,(-k)}-\varphi^\pi\rVert_2
\]
The right side is $o_P(1)$ by our analysis on Term (C), so by Markov's inequality applied conditionally on $\mathcal I_{-k}$, $\pr_{n_k}|\widehat\varphi^{\pi,(-k)}-\varphi^\pi|=o_P(1)$ for each $k$.

It remains to show $\pr_n(\varphi^\pi)^2\to\sigma^2_{\RTC}$. The terms $\varphi^\pi(O_i)^2$ are i.i.d.\ and bounded by $C_\rho^2$, so by the law of large numbers $\pr_n(\varphi^\pi)^2\to\E[(\varphi^\pi)^2]$ in probability. 
Therefore $\hat\sigma^2_{\RTC}=\pr_n(\widehat\varphi^\pi)^2\to\sigma^2_{\RTC}$ in probability.

\end{proof}

\subsection{\texorpdfstring{Proof of \cref{thm:budget-regret}}{}}\label{pf:budget-regret}
\begin{proof}
Decompose as follows.
\begin{align*}
  \bigl[J(\pi^\star_{X,\K^\star})-c(\K^\star)\bigr]-\bigl[J(\hat\pi_{X,\hat\K})-c(\hat\K)\bigr] &=\bigl[\hat J(\pi^\star_{X,\K^\star})-c(\K^\star)\bigr]-\bigl[\hat J(\hat\pi_{X,\hat\K})-c(\hat\K)\bigr]\\
  &\qquad+\bigl[J(\pi^\star_{X,\K^\star})-\hat J(\pi^\star_{X,\K^\star})\bigr]+\bigl[\hat J(\hat\pi_{X,\hat\K})-J(\hat\pi_{X,\hat\K})\bigr]\\
  &\quad\leq\bigl[J(\pi^\star_{X,\K^\star})-\hat J(\pi^\star_{X,\K^\star})\bigr]+\bigl[\hat J(\hat\pi_{X,\hat\K})-J(\hat\pi_{X,\hat\K})\bigr]
\end{align*}
where $\bigl[\hat J(\pi^\star_{X,\K^\star})-c(\K^\star)\bigr]-\bigl[\hat J(\hat\pi_{X,\hat\K})-c(\hat\K)\bigr]$ is non-positive because $(\hat\K,\hat\pi_{X,\hat\K})$ maximizes $\hat J(\pi)-c(\K)$.

For each remaining term, apply the decomposition from \cref{thm:regret}, to get equivalent efficient influence function, remainder, and cross-fitting terms: $(\mathrm E)+(\mathrm R)+(\mathrm C)$. Term (R) is controlled at the two policies in the same manner as \cref{thm:regret}, yielding $o_P(n^{-1/2})$.
Terms (E) and (C) require bounding the following, 
\begin{align*}
  \max_{\K:\,c(\K)\leq b}\ \sup_{\pi\in\Pi_{X,\K}}|(\mathrm E)(\pi)|+\max_{\K:\,c(\K)\leq b}\ \sup_{\pi\in\Pi_{X,\K}}|(\mathrm C)(\pi)|.
\end{align*}
empirical processes evaluated at a policy in $\bigcup_{\K:\,c(\K)\leq b}\Pi_{X,\K}$. 

For any policy determined by some $k \in \K$, the maximal inequality of \cref{thm:regret} gives a bound of $C_\rho\mathcal J_\K/\sqrt n$. Then, we apply a union bound over the $N_b$ options to get $C_\rho\sqrt{\log N_b}/\sqrt n$, yielding our claim.

\end{proof}

\section{Additional  Details for Numerical Experiments}
\label{app:sim-details}

\subsection{Synthetic data}

\paragraph{Setup}
We vary the overall sample size $n$, time periods $T$ and strength of runtime confounding $c$ according to
\[
n\in\{2000,8000,32000\},\qquad
T\in\{2,5,10\},\qquad
c\in\{0,1,2\}.
\]
Letting $\ell(z)=(1+e^{-z})^{-1}$, we additionally consider two scenarios dictating the degree of dependence on prior $t$ in our sequential data generation:
\[
\mathcal L_t(F)=
\begin{cases}
F_t, & \text{first-order Markov},\\[3pt]
\displaystyle
\frac{\sum_{k=0}^{t-1}2^{-k}F_{t-k}}
     {\sum_{k=0}^{t-1}2^{-k}},
& \text{fading memory}.
\end{cases}
\]

\paragraph{DGP}
Throughout, we assume $X_t\in\mathbb R^{20}$ is available at runtime, while
$U_t\in\mathbb R^3$ is only available in the historical (logged) data. The data are generated according to 
\begin{align*}
S_t&=(X_t,U_t),&
H_t&=(S_1,A_1,R_1,\ldots,S_t),\\
X_1&\sim N_{20}(0,I_{20}),&
U_1&\sim N_3(0,I_3),\\
U_{t+1}
&=0.85U_t+\sqrt{1-0.85^2}\,\xi_{t+1},&
\bar u_t&=\frac{1}{\sqrt3}\sum_{j=1}^3 U_{tj},\\
g_t&=0.1+0.8\ell(a^\top X_t+c\bar u_t),&
A_t\mid H_t&\sim\operatorname{Bernoulli}(g_t),\\
m_t&=0.1+b^\top\mathcal L_t(X)+\mathcal L_t(\bar u)\\
&\quad +(2A_t-1)
 \{d^\top\mathcal L_t(X)+0.4\mathcal L_t(\bar u)\}-0.1A_t,\\
R_t\mid H_t,A_t&\sim\operatorname{Bernoulli}\{\ell(m_t)\},\\
X_{t+1}
&=0.6\mathcal L_t(X)+v\mathcal L_t(2A-1)
 +w\mathcal L_t(\bar u)+0.6\varepsilon_{t+1},\\
\varepsilon_t&\overset{\mathrm{iid}}{\sim}N_{20}(0,I_{20}),&
\xi_t&\overset{\mathrm{iid}}{\sim}N_3(0,I_3).
\end{align*}
Above, we set
\begin{align*}
a&=(0.35,-0.25,0.20,0.15,-0.10,0,\ldots,0),\\
b&=(0.30,0.20,-0.20,0.15,0.10,0,\ldots,0),\\
d&=(0.35,-0.20,0.15,0,\ldots,0),\\
v&=(0.20,-0.10,0.10,0,\ldots,0),\\
w&=(0.20,0.10,0,0.10,0,\ldots,0).
\end{align*}

\paragraph{Policies}
We consider the following parametric policy classes:
\begin{align*}
\Pi_X&=\{\pi_\theta:\pi_\theta(1\mid X_t)
 =\ell(\theta_0+\theta_X^\top X_t),\ \|\theta\|_\infty\leq2\},\\
\Pi_{\mathrm{full}}
&=\{\pi_\vartheta:\pi_\vartheta(1\mid X_t,U_t)
 =\ell(\vartheta_0+\vartheta_X^\top X_t+\vartheta_U^\top U_t),
 \ \|\vartheta\|_\infty\leq2\},
\end{align*}
where we restrict the maximum coefficients to limit instances of poor overlap.

\paragraph{Comparator methods}
\[
\begin{array}{lll} \toprule
\text{Method} & \text{Adjustment} & \text{Policy class}\\ \midrule
\text{RTC-DR} & \text{Full history, including }U & \Pi_X\\ 
\text{Naive DR} & \text{History excluding }U & \Pi_X\\
\text{Full-information reference} & \text{Full history, including }U
 & \Pi_{\mathrm{full}} \\
 \bottomrule
\end{array}
\]

\paragraph{Nuisance fitting}
We let $\mathcal{D}_L$ and $\mathcal{D}_E$  denote the  non-overlapping indices used for policy learning and evaluation, respectively. We set  $|\mathcal D_L|=|\mathcal D_E|=n/2$, generating a random split for each Monte Carlo iteration. Within each split, we use 5-fold cross-fitting, where all nuisance learners use complete history available, while naive discards variables not available at runtime. 

The terminal $Q$ function is estimated with random forests, while earlier-period $Q$ functions use ridge regression and propensity scores use logistic elastic net. All tuning parameters are selected through inner cross validation. Estimated propensity scores are clipped to $[0.01,0.99]$.

\paragraph{Policy Fitting}

To avoid unstable policy selection, we split $\mathcal{D}_L$ 75/25 and choose policies which maximize the penalized loss \citep{luckett2020estimating}
\[
\hat J_L(\pi_\theta)/T - \lambda ||\theta_W||_2^2/2,
\]
where $\lambda$ is selected through cross-validation and $\theta_W$ contains the policy slope coefficients.

\paragraph{Operating Characteristics}

Fix a learned policy $\hat \pi_0^{(b)}$ from a generic method for a fixed Monte Carlo iteration $b$. We compute $J(\hat \pi_0^{(b)})$ by simulating 500,000 trajectories, each of length $T$, under our known DGP using $\hat \pi_0$. Letting $\hat J_E(\hat \pi^{(b)})$ denote the estimated value of this learned policy based on the evaluation fold $\mathcal{D}_E$, we report
\begin{align*}
    \text{Bias} &= B^{-1} \sum_{b=1}^B \{\hat J_E(\hat \pi_0^{(b)})   -J(\hat \pi_0^{(b)})\} \\
     \text{RMSE} &=\sqrt{ B^{-1} \sum_{b=1}^B \{\hat J_E(\hat \pi_0^{(b)})   -J(\hat \pi_0^{(b)})\}^2 }
\end{align*}

Letting $\hat \pi_\text{full}^{(b)} \in \Pi_\text{full}$ be the policy learned by the full-information learning, and letting $\hat \pi_0$ be the policy learned by either the RTC-DR or naive approach, we estimate regret according to 
\[
\widehat{\mathrm{Regret}}
=
\frac{1}{B}\sum_{b=1}^B
\left\{
J(\hat\pi_{\mathrm{full}}^{(b)})
-
J(\hat\pi_0^{(b)})
\right\}.
\]

\paragraph{Variable acquisition under budget} We consider our baseline setting with $T=5, c=1, n=8000$ and fading memory. To evaluate a scenario where the candidate runtime confounders provide different benefits, we alter the earlier reward equation so that
\begin{align*}
    R_t \mid H_t, A_t &\sim \text{Bernoulli}\{\ell(m_t)\}, \\
    m_t &= 0.1 + b^\top \mathcal{L}_t(X) + \frac{1.5 \mathcal{L}_t(U_1) + 0.75 \mathcal{L}_t(U_2)}{\sqrt{1.5^2 + 0.75^2}} \\
    &+ (2A_t-1)\left\{d^\top \mathcal{L}_t(X) + 0.4 \frac{1.5 \mathcal{L}_t(U_1) + 0.75 \mathcal{L}_t(U_2)}{\sqrt{1.5^2 + 0.75^2}} \right\} - 0.1 A_t.
\end{align*}
Relative to the baseline design, where $\bar u_t$ impacted rewards, only $U_1$ and $U_2$ directly affect rewards; $U_3$ still affects state transitions.

\section{Additional Experiments Results}
\label{app:sim-results}

\subsection{Synthetic Data}

Additional results for varying combinations of $n, T$ and $c$ are reported below.

\begin{figure}
    \centering
    \includegraphics[width=0.9\linewidth]{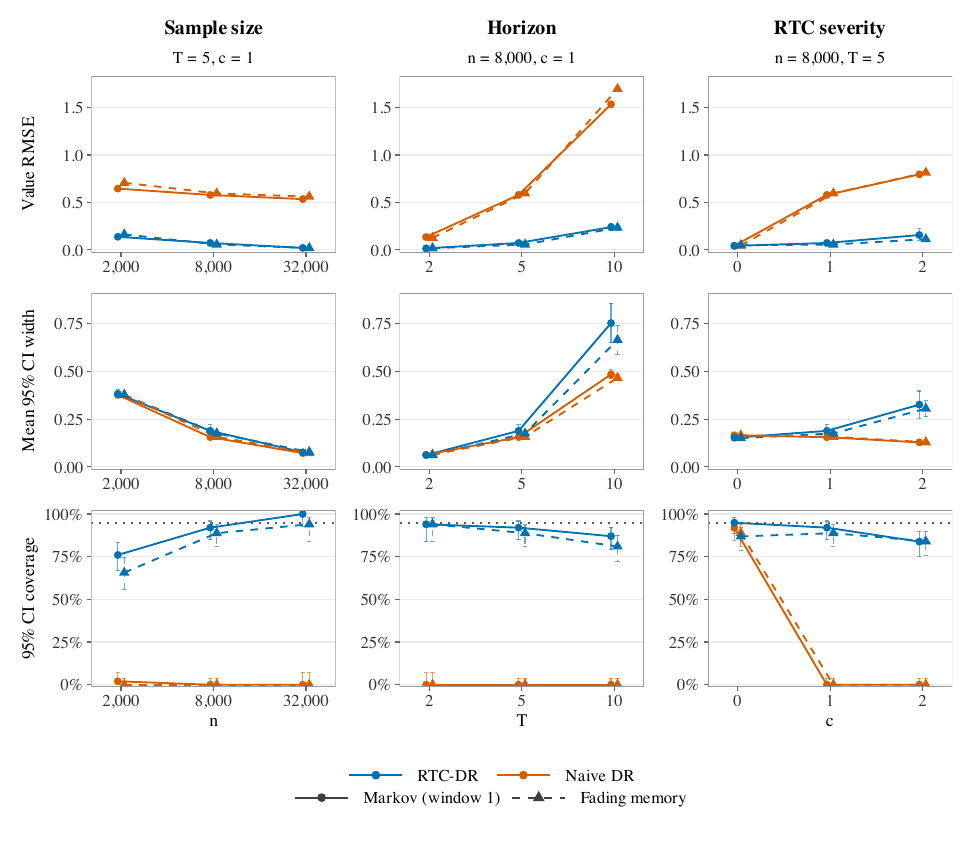}
    \caption{Operating characteristics for learned policy evaluation under additional combinations of $n$, $T$ and $c$.}
    \label{fig:policy-eval}
\end{figure}

\begin{figure}
    \centering
    \includegraphics[width=0.85\linewidth]{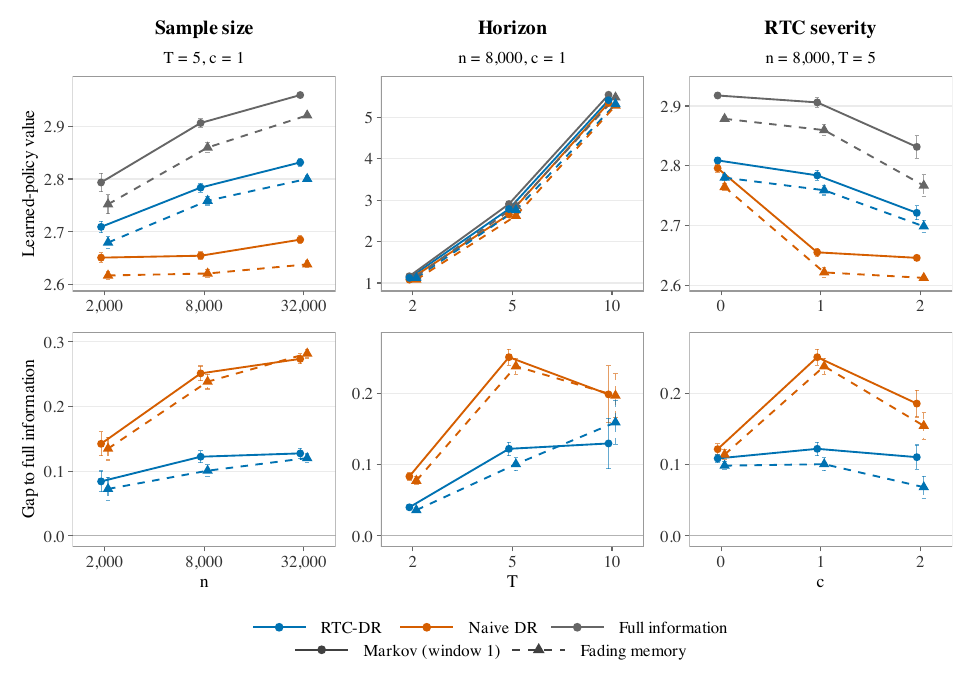}
    \caption{Learned policies values, and regret relative to the full-information oracle learner, under additional combinations of $n$, $T$ and $c$.}
    \label{fig:policy-evaltwo}
\end{figure}

\clearpage
\subsection{ADCB Data}

\section{Additional Details for ADCB Exercise}
\label{app:adcb-details}

\paragraph{Data} We generate independent cohorts with $T=5$ using the ADCB simulator \citep{kinyanjui2022adcb}, specifying the Hernandez behavioral policy, which uses patient demographics Mini-Mental State Examination (MMSE) scores,  CDRSB Clinical Dementia Rating – Sum of Boxes (CDRSB) scores, and previous treatments \citep{kinyanjui2022adcb}. We use the default exploration parameter of 0.1. 

At each time point, rewards $R_t$ are negative cognitive outcomes (ADAS-Cog13 scores). Table \ref{tab:adcb-availability} below provides information on the simulator's covariates. The historical data contains information on basic demographic information, as well as previous diagnoses and treatments, 4 biomarkers (TAU, PTAU, FDG, AV45), disease subtype, and cognitive function test scores (MMSE and CDRSB). Rewards are negative cognitive outcomes. We assume all 4 biomarkers,  MMSE and disease subtype are unavailable during deployment.

\begin{table}[ht]
\centering
\begin{tabular}{p{0.28\linewidth}p{0.63\linewidth}}
\toprule
Availability & Variables \\
\midrule
Historical and runtime $(X_t) $&
Age, CDRSB, education, ethnicity, race, sex, marital status,
previous diagnosis, and past treatments and rewards \\
\addlinespace
Historical only $(U_t)$&
MMSE, TAU, PTAU, FDG, AV45, and disease subtype \\
\bottomrule
\end{tabular}
\caption{ADCB covariates.}
\label{tab:adcb-availability}
\end{table}

We restrict deployment data so that the cognitive assesment score MMSE,the biomarkers TAU, PTAU, FDG, AV45, and disease subtype are observed historically but unavailable at runtime. 

\paragraph{Policy class} We consider policies of the form
\[
\pi_\theta(a \mid X_t) = \frac{1}{2} \bar p_a + \frac{1}{2}\frac{\exp\{\theta_{0a} + \theta_{Xa}^\top X_t\}}{\sum_{j=0}^7\exp\{\theta_{0j} + \theta_{Xj}^\top X_t\}}
\]
where for each $a$ $\bar p_a = \E[\pi_b(a\mid \cdot)]$ is the probability of taking action $a$ under the behavioral Hernandez policy. Restriction to this class limits instances of poor overlap, and similar policy classes are central to the policy improvement literature \citep{kallus2018confounding}. Analogous to our numerical simulations, the full-information oracle considers the policy class
\[
\pi_\vartheta(a \mid S_t) = \frac{1}{2} \bar p_a + \frac{1}{2}\frac{\exp\{\vartheta_{0a} + \vartheta_{Xa}^\top X_t + \vartheta_{Ua}^\top U_t\}}{\sum_{j=0}^7\exp\{\vartheta_{0j} + \vartheta_{Xj}^\top X_t +\vartheta_{Uj}^\top U_t\}}.
\]

\paragraph{Nuisance fitting} We use 20{,}000 patients for policy learning and evaluation separately. We employ the same penalization to the policy function used in our synthetic data experiments, where for each learning fold we similarly use cross validation to select the policy tuning parameter $\lambda$. 

Analogous to our synthetic experiments, we construct estimators through cross fitting, where the terminal $Q$ function is fit using random forests whereas earlier period $Q$ functions are fit using ridge regression, and propensity scores through multinational elastic net.  Following \cite{thomas2016data}, we use Hajek normalization to stabilize estimated weights for all three estimators.

\paragraph{Variable acquisition under budget} We vary the selection cost per variable in $\{0,0.05, \ldots, 10\}$. For simplicity, we only consider subsets $c(\mathcal{K})$ which select a single additional runtime variable, so that for a given cost we effectively choose between selecting no additional variables, or one of the variables treated unavailable at runtime. In turn, for each method we fit seven total policies: baseline plus one for each of the six candidate variables initially unavailable at runtime.

\end{document}